\documentclass{article}
\usepackage[hyperfootnotes=false]{hyperref}
\usepackage{xurl}

\usepackage[preprint]{icml2026}
\makeatletter
\renewcommand{\ICML@preprint}{\textit{Preprint.}}
\makeatother
\hypersetup{
  pdfsubject={Research preprint},
  pdfauthor={Shanfeng Huang, Rongqian Zhang, Zhou Fang, Song Xiao, Hai Du},
  pdfkeywords={Mixture-of-Experts, expert compression, pruning, merging, PRIME}
}

\usepackage{times}
\usepackage{latexsym}

\usepackage[T1]{fontenc}
\usepackage[utf8]{inputenc}

\usepackage{microtype}
\usepackage{graphicx}

\usepackage{subcaption}
\usepackage{booktabs}
\usepackage{amsmath}
\usepackage{amssymb}
\usepackage{mathtools}
\usepackage{amsthm}
\usepackage{enumitem}
\usepackage{xcolor}
\usepackage{multirow}
\usepackage[capitalize,noabbrev]{cleveref}
\usepackage{caption}

\theoremstyle{plain}
\newtheorem{theorem}{Theorem}[section]

\theoremstyle{definition}

\theoremstyle{remark}
\newtheorem{remark}[theorem]{Remark}

\usepackage[textsize=tiny]{todonotes}

\icmltitlerunning{PRIME: Frequency-Aware Unified Pruning and Merging for MoE Compression}

\begin{document}
\twocolumn[
  \icmltitle{PRIME: Frequency-Aware Unified Pruning and Merging for MoE Compression}
  \begin{icmlauthorlist}
    \icmlauthor{Shanfeng Huang}{baidu}
    \icmlauthor{Rongqian Zhang}{baidu}
    \icmlauthor{Zhou Fang}{baidu}
    \icmlauthor{Song Xiao}{baidu}
    \icmlauthor{Hai Du}{baidu}
  \end{icmlauthorlist}
  \icmlaffiliation{baidu}{Baidu Inc}
  \icmlcorrespondingauthor{Authors}{huangshanfeng@baidu.com}
  \vskip 0.3in
]
\printAffiliationsAndNotice{}

\begin{abstract}  Mixture-of-Experts (MoE) models have emerged as an effective paradigm for scaling large language models, yet their massive parameter footprint poses severe memory challenges for deployment. Existing compression methods typically treat expert pruning and merging as disjoint operations and rely on greedy heuristics, often leading to irreversible information loss or representational pollution. In this paper, we propose \textbf{PRIME}, a unified optimization framework for MoE compression via adaptive \textbf{PR}un\textbf{I}ng and \textbf{ME}rging. PRIME casts the discrete prune-merge assignment step as a global Integer Linear Programming over a frequency-weighted surrogate objective, jointly selecting pruning and merging decisions while accounting for expert usage. We motivate the frequency-aware objective through an error-bound analysis under fixed representative assignments, and show that the formulation favors high-frequency experts as centroids, promoting routing stability. Extensive experiments on Qwen3-30B-A3B and Qwen3.5-122B-A10B demonstrate that PRIME achieves strong accuracy retention across knowledge, coding, math, and long-context tasks, while serving benchmarks of merged models quantify the memory and throughput benefits of expert compression in vLLM.
\end{abstract}

\begin{figure*}[t]
    \centering
    \includegraphics[width=\textwidth]{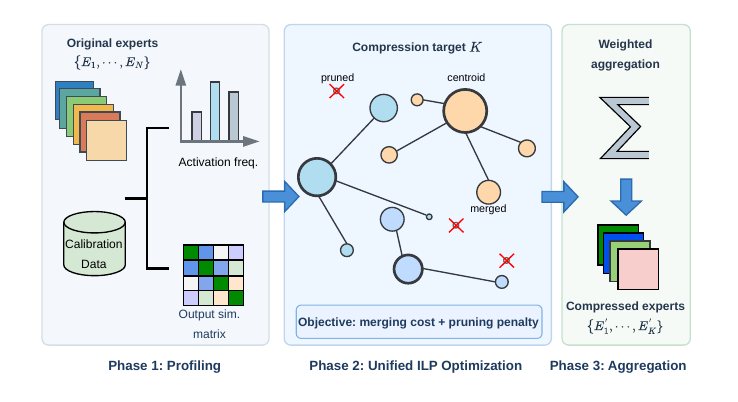}
    \caption{Overview of the PRIME framework. (1) Profiling: extract activation frequencies and pairwise distances. (2) Unified ILP: jointly determine prune-merge decisions. (3) Weighted aggregation: synthesize compressed experts via frequency-weighted averaging.}
\label{fig:framework}
\end{figure*}

\section{Introduction}
The rapid progress of Large Language Models (LLMs) has been strongly driven by Mixture-of-Experts (MoE) architectures, which decouple model capacity from inference cost through sparse activation. By activating only a sparse subset of experts per token, MoEs achieve performance comparable to dense models via conditional computation. However, this efficiency comes at the cost of immense memory footprints, as the full set of experts typically requires residence in memory. This bottleneck hinders the deployment of MoEs on commodity hardware and edge devices.

Consequently, compressing MoE models has become a critical research frontier.
Beyond conventional compression techniques widely used in dense models, such as quantization \citep{dettmers2022llmint8}, sparsification \citep{frantar2023sparsegpt} and low-rank factorization \citep{hu2022lora}, MoE architectures introduce unique opportunities for structural compression, e.g. expert pruning which removes under-utilized or redundant experts \citep{chen2022task, lu2024not, xie2024moepruner, huang2024mcmoe, guo2025cluster, lasby2025reap}, and expert merging which consolidates semantically similar experts into fewer, more compact ones \citep{li2023merge, miao2025mergemoe, chen2025hcsmoe, zhang2025diversifying}. However, current compression paradigms exhibit inherent limitations. On the one hand, existing optimization strategies are commonly bifurcated, as pruning and merging are regarded as mutually exclusive operations or are executed as sequentially decoupled stages, thereby precluding the exploration of a unified joint decision space. On the other hand, these approaches ignore expert-frequency effects by relying solely on similarity metrics and treating all experts equally during clustering, disregarding that distortions in high-frequency experts accumulate over tokens and cause much larger inference errors than even greater deviations in rarely used experts.

In this work, we propose PRIME, a frequency-aware unified pruning and merging framework for MoE compression as shown in Figure \ref{fig:framework}. The contributions are: \begin{itemize}[leftmargin=*, noitemsep, topsep=0pt, itemsep=0pt] 
\item \textbf{Unified Pruning-Merging via ILP.} We formulate pruning and merging decisions within a single ILP over a frequency-weighted assignment objective, obtaining the optimum of the stated surrogate rather than relying on greedy local decisions. 
\item \textbf{Frequency-Aware Error Motivation.} We derive an error-bound analysis that motivates weighting assignment costs by expert activation frequency. We further derive a sufficient condition showing that high-frequency experts are naturally favored as centroids, contributing to routing stability.
\item \textbf{Deployment Analysis.} Benchmarks of merged models on vLLM quantify the serving benefits of expert compression: reducing the expert count by 25\% lowers weight memory by 23.7\%, expands KV cache capacity by 2.20$\times$, and improves output throughput by 1.74$\times$ in a representative memory-constrained workload.
\end{itemize}

\section{Related Works}
\label{sec:related_works}


\paragraph{Expert Pruning.}
Expert pruning aims to remove redundant experts using importance estimates. Early methods rely on static metrics from pre-trained weights or calibration statistics. MoE-Pruner \cite{xie2024moepruner} combines router weights and activation magnitudes, while \citet{lu2024not} prune experts with low global usage frequency. \citet{guo2025cluster} further introduces structural awareness by clustering experts and removing redundant clusters. Although efficient, these methods are static and context-agnostic, and may discard long-tail experts that are useful for domains such as math or coding. Recent works introduce learnability to improve pruning decisions. DiEP \cite{bai2025diep} learns layer-specific sparsity rates through differentiable pruning. More recently, REAP \cite{lasby2025reap} uses a router-weighted activation criterion to minimize reconstruction error. Nevertheless, pruning remains a binary retain-or-discard operation, which limits the optimization space and cannot consolidate information from useful but low-frequency experts.

\paragraph{Expert Merging.}
Unlike pruning, expert merging seeks to compress the model by consolidating multiple experts into a smaller set of representatives, theoretically preserving the collective knowledge. This direction is grounded in the theory of Linear Mode Connectivity \cite{frankle2020linear} and Model Soups \cite{wortsman2022model}, which suggest that parameters residing in the same basin of attraction can be averaged.
Driven by these theoretical insights, recent research has sought to extend parameter merging techniques to experts within MoE architectures. Early explorations, such as EEP \cite{liu2024efficient}, utilized merging primarily to recover model capacity, aggregating the weights of pruned experts into retained ones. Subsequently,  \citet{li2023merge} further established the feasibility of this paradigm with MC-SMoE, demonstrating that multiple experts can be effectively consolidated into compact representatives without retraining.
Building on this foundation, recent state-of-the-art approaches have adopted clustering algorithms to guide the consolidation process. HC-SMoE \cite{chen2025hcsmoe} employed hierarchical clustering to iteratively identify and average functionally similar experts without retraining. While efficient, this greedy process inherently prioritizes local stepwise decisions, which may miss lower-cost assignments under a joint objective. 
\paragraph{The Gap.}
Current paradigms largely treat pruning and merging as disjoint or sequential operations, limiting the joint decision space needed for stronger trade-offs. Moreover, standard geometric clustering ignores expert activation frequencies, treating critical high-traffic experts identically to rarely used ones, disproportionately inflating the aggregate inference error. Bridging this gap, PRIME introduces a frequency-aware unified optimization framework that jointly determines prune-merge assignments.

\section{Methodology}
\label{sec:method}
In this section, we introduce PRIME, a unified optimization framework for MoE compression. We cast the discrete pruning and merging assignment step as a frequency-aware optimization problem that minimizes a surrogate for compression-induced error. Expert distances are measured by output-based distances as discussed in HC-SMoE \cite{chen2025hcsmoe}. This yields a frequency-aware ILP whose surrogate optimum provides an effective strategy for preserving high-traffic routing paths.

\subsection{Problem Definition}
Let the original expert set be $\mathcal{E} = \{E_1, \dots, E_N\}$. Our goal is to compress the expert set $\mathcal{E}$ into a smaller set $\mathcal{E}'$ of size $K$ ($K < N$) without retraining. This formulation induces a unified decision process in which each original expert $E_i$ is assigned to one of the following roles:
\begin{itemize}[leftmargin=*, noitemsep, topsep=0pt, parsep=0pt, partopsep=0pt]
\item \textbf{Centroid expert}: selected as the representative of a cluster and updated through weighted aggregation with other experts assigned to the same cluster;
\item \textbf{Merged expert}: assigned to a certain centroid expert, contributing its parameters to the centroid via weighted aggregation;
\item \textbf{Pruned expert}: identified as a geometrically isolated outlier or a low-usage expert whose merging cost exceeds the pruning penalty, leading to its removal from the model.
\end{itemize}

\subsection{Frequency-Aware Surrogate Optimization}
\label{sec:ilp_formulation}

We formulate the assignment stage of MoE compression as a discrete optimization problem over a frequency-weighted surrogate objective. Let $f_i$ be the normalized activation frequency of expert $E_i$ on the calibration dataset $\mathcal{D}_{cal}$. Conventional clustering methods typically rely solely on geometric similarity, treating all experts as equivalent nodes in the metric space. This assumption is flawed in the context of MoE inference since the harm caused by merging or pruning an
expert is highly correlated with its usage frequency. A minor distortion in a high-frequency expert accumulates rapidly across tokens, disproportionately inflating the total error compared to significant distortions in rarely used experts.

To address this, we incorporate $f_i$ as a scaling factor. Appendix~\ref{app:error_bound} motivates this weighting by showing that $\sum f_i \cdot d(\cdot, \cdot)$ upper-bounds inference error; the specific output-space metric $D_{ij}$ used in the ILP is an empirical instantiation rather than a direct consequence of the bound (Remark~\ref{rem:metric_gap}). Moreover, as a larger $f_i$ increases the penalty for pruning or reassigning expert $E_i$, this mechanism encourages the solver to anchor high-frequency experts as centroids, thereby discouraging the aggregation of multiple high-traffic experts and promoting routing stability (Theorem \ref{th:sufficient}, Appendix~\ref{app:opt_condition}).

Formally, we define the following binary decision variables:
\begin{itemize}[leftmargin=*, noitemsep, topsep=0pt, parsep=0pt, partopsep=0pt]
\item $z_j \in \{0, 1\}$: Equal to 1 if  $E_j$ is selected as a centroid.
\item $x_{ij} \in \{0, 1\}$: Equal to 1 if $E_i$ is merged into centroid $E_j$.
\item $y_i \in \{0, 1\}$: Equal to 1 if  $E_i$ is pruned.
\end{itemize}
Guided by this frequency-aware surrogate, we formulate the following ILP to minimize total weighted distance alongside pruning penalties:
\begin{equation}
    \min \quad \sum_{i=1}^N \sum_{j=1}^N f_i D_{ij} x_{ij} +  \lambda\sum_{i=1}^N f_i  y_i \label{eq:objective}
\end{equation}
\begin{align}
    \text{s.t.}\quad \sum_{j=1}^N x_{ij} + y_i &= 1, \quad \forall i \in \{1,\dots,N\} \label{eq:assign} \\
    x_{ij} &\le z_j, \quad \forall i, j \in \{1,\dots,N\} \label{eq:consistency} \\
    \sum_{j=1}^N z_j &= K \label{eq:budget} \\
    z_j &\leq 1-y_j, \quad \forall j \in \{1,\dots,N\} \label{eq:self2} \\
    x_{jj} &= z_j, \quad \forall j \in \{1,\dots,N\}. \label{eq:self}  
\end{align}
Here, $D_{ij}$ denotes the normalized Euclidean distance between the mean outputs of $E_i$ and $E_j$ on the calibration set---an empirical output-space proxy that is motivated by, but not identical to, the parameter-space metric in Theorem~\ref{th:error_bound} (see Remark~\ref{rem:metric_gap}). $\lambda$ denotes a constant penalty associated with pruning an expert, which can equivalently be interpreted as the tolerance threshold on the maximum allowable distance between experts for them to be merged. The weighting term $f_i$ ensures that high-frequency experts incur a larger cost when pruned or assigned to distant centroids.
Equation \ref{eq:assign} ensures that every expert is either merged or pruned. Equation \ref{eq:consistency} dictates that an expert can only be merged into an active centroid. Equation \ref{eq:budget} enforces the target model size where $K$ is the target number of experts after compression. 
Equations~\ref{eq:self2} and~\ref{eq:self} enforce self-assignment for every center, ensuring that centers cannot be pruned and preventing phantom centers.

This ILP is solved using off-the-shelf solvers (e.g., Gurobi or CBC), yielding an optimal solution to the surrogate assignment objective in Eq.~\ref{eq:objective}. This does not imply global optimality of the final compressed model under true downstream loss; rather, it provides a tractable assignment step that avoids the local decisions of heuristic clustering. We further analyze the solution structure in Appendix~\ref{app:opt_condition}.

\subsection{Experts Aggregation}
\label{sec:aggregation}
Once the optimal assignment $x^*_{ij}$ is obtained, the parameters $W_j^{'}$ of the reserved centroid $E_j$ are updated by aggregating all experts $E_i$ in its cluster $\mathcal{C}_j = \{i \mid x^*_{ij}=1\}$. We adopt frequency-weighted averaging which has been empirically shown in HC-SMoE \cite{chen2025hcsmoe} to be highly effective compared to common aggregation strategies:
\begin{equation}\label{eq:merge}
    W_j^{'} = \frac{\sum_{i \in \mathcal{C}_j} f_i W_i}{\sum_{i \in \mathcal{C}_j} f_i}.
\end{equation}
This aggregation rule is a post-assignment construction step rather than part of the ILP objective itself. It is nevertheless aligned with the frequency-aware assignment: high-frequency experts dominate the final centroid parameters, reducing drift from frequently used routing paths. Thus, Eq.~\ref{eq:objective} selects prune-merge assignments, while Eq.~\ref{eq:merge} instantiates the selected clusters.

\subsection{Router Reconstruction}
\label{sec:router_reconstruction}
Expert aggregation requires a corresponding router update. Let $W_r\in\mathbb{R}^{N\times d}$ be a linear softmax router with logits $\ell_i(h)=w_i^\top h$ for token representation $h\in\mathbb{R}^d$. Pruning simply deletes router rows; softmax over the remaining logits renormalizes the surviving probabilities. Merging instead requires summing member probabilities, which generally cannot be reproduced by averaging router rows.

\paragraph{Prune-aware grouped targets.}
Let $\mathcal{P}=\{i\mid y_i^*=1\}$ denote the pruned experts, and relabel the $K$ selected clusters from Section~\ref{sec:aggregation} as $\mathcal{G}_1,\ldots,\mathcal{G}_K$. These groups partition the surviving experts. We first remove pruned experts from the normalization and then aggregate probability mass within each group:
\begin{equation}
\begin{aligned}
p_i^{\mathrm{surv}}(h)
 &= \frac{\exp(\ell_i(h))}{\sum_{u\notin\mathcal{P}}\exp(\ell_u(h))},\quad i\notin\mathcal{P},\\
q_j(h)&=\sum_{i\in\mathcal{G}_j}p_i^{\mathrm{surv}}(h).
\end{aligned}
\label{eq:grouped_routing}
\end{equation}
Equivalently, $q(h)$ is the softmax of the grouped logits
\begin{equation}
\widetilde{\ell}_j(h)
=\log\sum_{i\in\mathcal{G}_j}\exp\!\left(w_i^\top h\right).
\label{eq:grouped_logits}
\end{equation}
Singleton groups retain their original router logits. For merged groups, log-sum-exp is generally nonlinear in $h$: exact grouped routing still requires the original surviving logits online, rather than only $K$ linear scores.

\paragraph{Offline linear reconstruction.}
We approximate these input-dependent targets on calibration tokens. Stack their hidden states in $H\in\mathbb{R}^{T\times d}$ and the grouped logits from Eq.~\ref{eq:grouped_logits} in $\widetilde{Z}\in\mathbb{R}^{T\times K}$. The compressed router is obtained by ridge regression:
\begin{equation}
W'_r=\arg\min_{W\in\mathbb{R}^{K\times d}}
\|HW^\top-\widetilde{Z}\|_F^2+\lambda_r\|W\|_F^2,
\label{eq:router_ridge}
\end{equation}
where $\lambda_r>0$ regularizes the fit and is distinct from the ILP pruning penalty $\lambda$. Its closed-form solution is
\begin{equation}
W'_r=\widetilde{Z}^{\top}H
\left(H^\top H+\lambda_r I_d\right)^{-1},
\label{eq:router_solution}
\end{equation}
computed through a linear solve. This calibration-time step uses no backpropagation or iterative optimizer and leaves the aggregated expert weights unchanged. The resulting $K\times d$ weight matrix retains the standard router layout for export in a Hugging Face checkpoint, without an online group-summation operator. The original router is needed only to construct calibration targets, not at inference. The reconstruction approximates grouped probabilities before top-$k$ selection; it does not guarantee identical sparse routing decisions or exact preservation of the original MoE outputs.

\section{Experiments}

\subsection{Experimental Setup}
To evaluate the compression efficacy and generalization boundaries of PRIME, we conduct extensive experiments on Qwen3-30B-A3B, a fine-grained MoE architecture with 128 experts per layer, 3 billion active parameters out of a total of 30 billion. All experiments are conducted on NVIDIA A800 (80GB) GPUs. Our framework is implemented using PyTorch and the Hugging Face Transformers library (v5.9.0). The ILP problem is solved by PuLP \cite{pulp2011} version 3.3.0 with its built-in PULP\_CBC\_CMD solver. Unless otherwise specified, we use a fixed set of hyperparameters, with $\lambda=0.25$, which are empirically observed to work well across all settings.

\textbf{Models.} We investigate performance boundaries by compressing Qwen3-30B-A3B to 96, 80, and 72 experts per layer, corresponding to reduction rates of 25\%, 37.5\%, and 43.75\%, respectively. To demonstrate scalability to industrial-level models, we further evaluate PRIME on Qwen3.5-122B-A10B, a significantly larger architecture with 256 experts per layer and 10 billion active parameters out of 122 billion total. We compress this model to 192, 160, and 144 experts (Section~\ref{sec:scalability}).

\textbf{Baselines.} We benchmark PRIME against HC-SMoE \citep{chen2025hcsmoe} and REAP \citep{lasby2025reap}, representing the current state-of-the-art in retraining-free expert merging and pruning, respectively. These two methods subsume earlier approaches in their respective paradigms: HC-SMoE has been shown to outperform prior merge methods (e.g., MC-SMoE), while REAP improves upon earlier pruning heuristics (e.g., MoE-Pruner, frequency-threshold pruning). Together they span the prune-merge spectrum against which a unified method should be evaluated. To ensure a controlled comparison, PRIME and HC-SMoE share the same calibration data, output-based distance metric, and frequency-weighted aggregation; the sole difference is the assignment strategy (ILP vs. greedy hierarchical clustering). REAP, as a prune-only method, does not involve distance computation or weight merging.

\textbf{Benchmarks.} We adopt a comprehensive evaluation suite spanning four capability dimensions: (1) \emph{Multiple Choice}: BoolQ \citep{clark2019boolq}, MMLU \citep{hendrycks2021mmlu}, OpenBookQA (OBQA) \citep{mihaylov2018openbookqa}, and RTE \citep{bentivogli2009rte}; (2) \emph{Long Context}: LongBench \citep{bai2024longbench}; (3) \emph{Code Generation}: HumanEval (HE) \citep{chen2021humaneval} and MBPP \citep{austin2021mbpp}; (4) \emph{Mathematical Reasoning}: GSM8K \citep{cobbe2021gsm8k}. This diverse suite enables us to assess whether compression preserves both knowledge retrieval and generative capabilities. For calibration data, we sample 32 sequences of 2,048 tokens each from a mixed corpus of C4 \citep{raffel2020t5}, Evol-CodeAlpaca \citep{luo2023wizardcoder}, and NuminaMath \citep{li2024numinamath} to cover diverse capability dimensions.

\subsection{Main Results and Performance Comparisons}
\label{sec:mainresults}

\begin{table*}[t]
\caption{Performance comparison on Qwen3-30B-A3B. $K$ = retained experts per layer.}
\label{tab:main_results_qwen}
\centering
\begin{small}
\begin{sc}
\resizebox{\textwidth}{!}{%
\begin{tabular}{c l ccccc c cc c c c}
\toprule
 & & \multicolumn{5}{c}{Multiple Choice} & \multicolumn{1}{c}{Long Context} & \multicolumn{3}{c}{Coding} & Math \\
\cmidrule(lr){3-7} \cmidrule(lr){8-8} \cmidrule(lr){9-11} \cmidrule(lr){12-12}
$K$ & Method & BoolQ & MMLU & OBQA & RTE & MC Avg & LongBench & HE & MBPP & Coding Avg & GSM8K \\
\midrule
128 & Original & 86.91 & 77.73 & 40.20 & 82.67 & 71.88 & 41.00 & 90.85 & 77.51 & 84.18 & 91.43 \\
\midrule
\multirow{3}{*}{96} & HC-SMoE & 84.98 & 62.32 & 38.60 & 77.62 & 65.88 & 39.38 & \textbf{90.85} & 74.07 & 82.46 & 90.22 \\
& REAP & 85.11 & 72.03 & \textbf{39.60} & 79.06 & 68.95 & 40.25 & 87.80 & 76.19 & 82.00 & \textbf{91.43} \\
& \textbf{PRIME (Ours)} & \textbf{85.96} & \textbf{74.89} & 39.40 & \textbf{81.59} & \textbf{70.46} & \textbf{41.75} & 87.20 & \textbf{77.78} & \textbf{82.49} & 91.21 \\
\midrule
\multirow{3}{*}{80} & HC-SMoE & 81.77 & 54.82 & 35.00 & 77.26 & 62.21 & 37.50 & \textbf{87.20} & 62.43 & 74.82 & 86.73 \\
& REAP & 82.39 & \textbf{66.64} & 36.60 & 77.98 & 65.90 & 37.13 & 81.71 & 58.47 & 70.09 & \textbf{89.84} \\
& \textbf{PRIME (Ours)} & \textbf{83.67} & 65.09 & \textbf{37.40} & \textbf{79.42} & \textbf{66.40} & \textbf{39.38} & 86.59 & \textbf{73.81} & \textbf{80.20} & 89.08 \\
\midrule
\multirow{3}{*}{72} & HC-SMoE & 78.93 & 50.27 & 34.60 & \textbf{76.17} & 59.99 & 35.25 & \textbf{80.49} & 61.64 & 71.07 & 54.74 \\
& REAP & 73.88 & 55.79 & 31.20 & 56.32 & 54.30 & 27.75 & 75.61 & 34.13 & 54.87 & 64.06 \\
& \textbf{PRIME (Ours)} & \textbf{80.55} & \textbf{59.60} & \textbf{36.20} & 65.34 & \textbf{60.42} & \textbf{39.13} & 78.66 & \textbf{65.61} & \textbf{72.13} & \textbf{87.49} \\
\bottomrule
\end{tabular}
}
\end{sc}
\end{small}
\end{table*}

\begin{table*}[t]
\caption{Ablation study on Qwen3-30B-A3B. Prune-Only: $\lambda=0$; Merge-Only: $\lambda=1$; w/o Frequency: $f_i=1$ for all experts.}
\label{tab:ablation}
\centering
\begin{small}
\begin{sc}
\resizebox{\textwidth}{!}{%
\begin{tabular}{c l ccccc c cc c c c}
\toprule
 & & \multicolumn{5}{c}{Multiple Choice} & \multicolumn{1}{c}{Long Context} & \multicolumn{3}{c}{Coding} & Math \\
\cmidrule(lr){3-7} \cmidrule(lr){8-8} \cmidrule(lr){9-11} \cmidrule(lr){12-12}
$K$ & Strategy & BoolQ & MMLU & OBQA & RTE & MC Avg & LongBench & HE & MBPP & Coding Avg & GSM8K \\
\midrule
\multirow{4}{*}{96} & Prune-Only & 84.13 & 71.31 & \textbf{39.60} & 80.14 & 68.80 & 40.63 & 82.93 & 76.46 & 79.69 & \textbf{91.58} \\
& Merge-Only & 85.66 & 69.66 & 39.00 & \textbf{81.59} & 68.98 & 41.13 & 80.49 & 75.93 & 78.21 & 90.98 \\
& w/o Frequency & 84.34 & 68.22 & 35.80 & 78.70 & 66.77 & 40.50 & 85.37 & 74.60 & 79.98 & 90.98 \\
& \textbf{PRIME} & \textbf{85.96} & \textbf{74.89} & 39.40 & \textbf{81.59} & \textbf{70.46} & \textbf{41.75} & \textbf{87.20} & \textbf{77.78} & \textbf{82.49} & 91.21 \\
\midrule
\multirow{4}{*}{80} & Prune-Only & 82.39 & 61.65 & 31.80 & 77.98 & 63.45 & 37.13 & 75.00 & 57.14 & 66.07 & 88.40 \\
& Merge-Only & 82.84 & 60.82 & 35.20 & 79.06 & 64.48 & 39.00 & 85.37 & 70.90 & 78.13 & 88.70 \\
& w/o Frequency & 80.95 & 59.81 & 34.40 & 75.45 & 62.65 & 38.75 & 83.54 & 67.99 & 75.76 & 84.76 \\
& \textbf{PRIME} & \textbf{83.67} & \textbf{65.09} & \textbf{37.40} & \textbf{79.42} & \textbf{66.40} & \textbf{39.38} & \textbf{86.59} & \textbf{73.81} & \textbf{80.20} & \textbf{89.08} \\
\midrule
\multirow{4}{*}{72} & Prune-Only & 73.88 & 50.80 & 30.20 & 56.32 & 52.80 & 28.13 & 73.78 & 35.71 & 54.75 & 76.95 \\
& Merge-Only & 78.62 & 54.74 & 33.20 & 63.54 & 57.53 & 38.00 & 76.22 & \textbf{69.31} & \textbf{72.77} & 85.60 \\
& w/o Frequency & 80.06 & 48.93 & 27.60 & \textbf{67.51} & 56.03 & 35.50 & 70.12 & 61.11 & 65.62 & 84.00 \\
& \textbf{PRIME} & \textbf{80.55} & \textbf{59.60} & \textbf{36.20} & 65.34 & \textbf{60.42} & \textbf{39.13} & \textbf{78.66} & 65.61 & 72.13 & \textbf{87.49} \\
\bottomrule
\end{tabular}
}
\end{sc}
\end{small}
\end{table*}

Table~\ref{tab:main_results_qwen} presents the evaluation results on Qwen3-30B-A3B across varying compression levels. PRIME achieves the best aggregate performance at every compression level, though individual metrics occasionally favor single-paradigm baselines.

At 25\% expert reduction ($K=96$), PRIME achieves the highest MC average of 70.46, compared with 68.95 for REAP and 65.88 for HC-SMoE, while maintaining a long-context score similar to the original model. We note that REAP matches the original on GSM8K (91.43) and HC-SMoE retains the original HumanEval (90.85), indicating that single-paradigm methods can retain specific strengths under mild compression. At $K=80$ (37.5\% reduction), the coding advantage becomes pronounced: PRIME preserves 80.20 coding average where both baselines degrade to 74.82 and 70.09, while still leading on MC average. REAP retains a slight edge on MMLU (66.64 vs.\ 65.09), indicating its pruning criterion better preserves certain capabilities at this level.

The gap becomes most pronounced at $K=72$ (43.75\% reduction). PRIME maintains 60.42 MC average, 72.13 coding average, 39.13 LongBench, and 87.49 GSM8K, while REAP collapses across all dimensions and HC-SMoE suffers catastrophic math degradation to 54.74 on GSM8K. Notably, PRIME at $K=72$ retains 87.49 on GSM8K, only 3.94 below the original, demonstrating exceptional mathematical robustness under aggressive compression.

Across all settings, a consistent pattern emerges: REAP occasionally preserves narrow advantages on specific tasks (e.g., MMLU, GSM8K) under mild compression, while HC-SMoE retains isolated coding strengths (HumanEval) due to its pure merge strategy avoiding expert removal. However, both baselines suffer catastrophic failures in at least one dimension under aggressive compression, whereas PRIME's primary advantage lies in balanced cross-dimension retention without such collapses.

\subsection{Ablation Study}
\label{sec:ablation}

To disentangle the contributions of individual components in PRIME, we conduct multidimensional ablation studies on Qwen3-30B-A3B. The quantitative results are summarized in Table~\ref{tab:ablation}.

\paragraph{Unified Adaptive Prune-and-Merge vs. Prune-Only / Merge-Only.}
We compare PRIME against two boundary conditions: \textit{Prune-Only} with $\lambda=0$ that forces pure elimination, and \textit{Merge-Only} with $\lambda=1$ that forces pure consolidation.
At $K=80$, PRIME achieves the best MC average of 66.40 and GSM8K of 89.08, outperforming both \textit{Prune-Only} at 63.45 / 88.40 and \textit{Merge-Only} at 64.48 / 88.70, while also leading on coding average (80.20 vs.\ 78.13 for Merge-Only and 66.07 for Prune-Only). The complementarity becomes clearer at $K=72$: \textit{Prune-Only} collapses on coding to a 54.75 average and long-context to 28.13, while \textit{Merge-Only} degrades substantially on MC to a 57.53 average and math to 85.60. PRIME still maintains good aggregate balance across dimensions. This demonstrates that the adaptive formulation effectively selects stronger per-layer prune-merge decisions, avoiding the catastrophic failures of either boundary approach.

\paragraph{Impact of Frequency-Awareness.}
We compare PRIME against a frequency-agnostic baseline that sets $f_i=1$ for all experts.
The impact grows with compression intensity. At $K=96$, removing frequency awareness reduces MC average from 70.46 to 66.77, and coding average from 82.49 to 79.98. At $K=80$, the degradation widens: MC average drops by 3.75 to 62.65, and math drops by 4.32 from 89.08 to 84.76. At the extreme $K=72$, the frequency-agnostic variant falls to 56.03 MC average and 65.62 coding average, while PRIME sustains 60.42 and 72.13 respectively. This confirms that frequency weighting is important for discouraging the solver from merging high-utility experts into less suitable clusters, with the effect amplified under aggressive compression where fewer centroid slots must be allocated to the most impactful experts.

\begin{figure}[t]
\centering
\centerline{\includegraphics[width=\columnwidth]{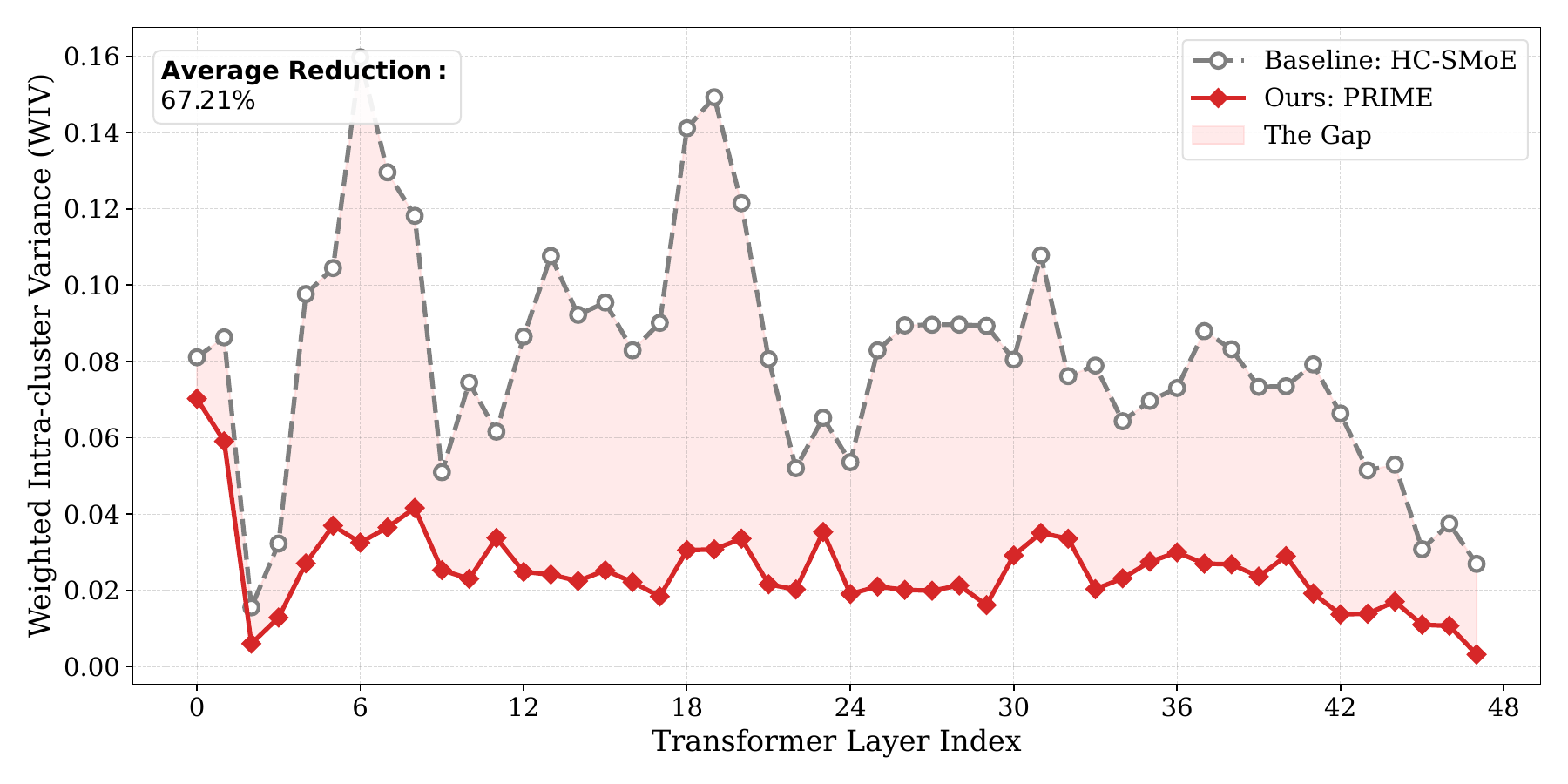}}
\caption{Layer-wise WIV comparison on Qwen3-30B-A3B. PRIME ($\lambda=1$) reduces variance by 67.21\% on average over HC-SMoE.}
\label{fig:wiv}
\end{figure}

\begin{figure}[t]
\centering
\includegraphics[width=\columnwidth]{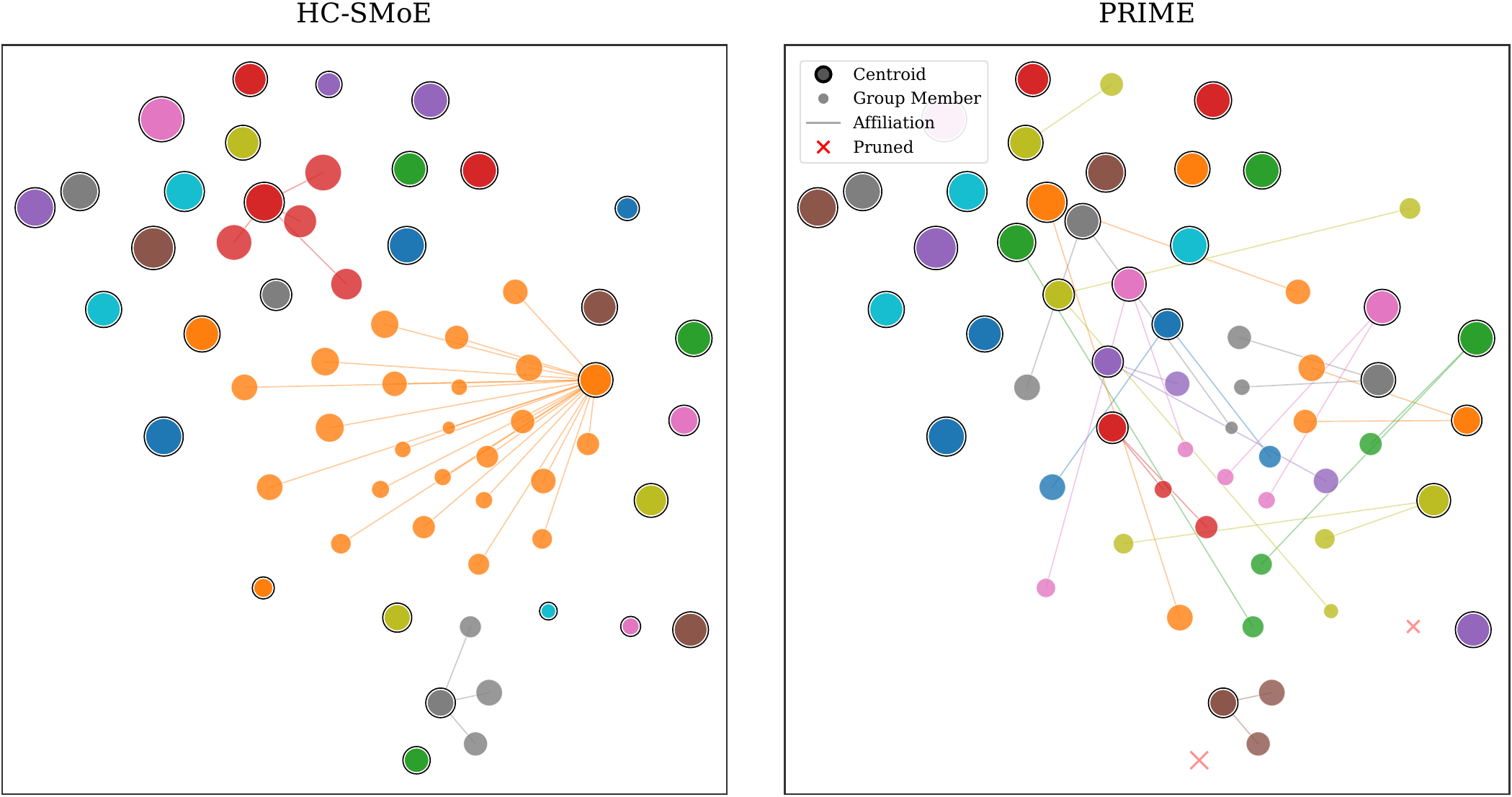}
\caption{t-SNE visualization of expert clustering on Qwen1.5-MoE (see Appendix~\ref{app:tsne_qwen3} for Qwen3). Left: HC-SMoE; Right: PRIME.}
\label{fig:tsne_viz}
\end{figure}

\subsection{Mechanistic Analysis}
\label{sec:analysis}

\textbf{Reducing Compression-Induced Distortion via Frequency-Aware Surrogate Optimization.}
Expert merging methods consolidate multiple experts into shared centroids, which can introduce failure modes such as feature dilution and routing collapse that are less common in pruning-based approaches. In this section, we analyze these merging-specific phenomena and compare PRIME primarily against HC-SMoE, the representative merge-based baseline.
To shed light on why PRIME achieves better results than greedy heuristics, we analyze how different compression strategies influence the behavior of expert groups.
Quantitatively, we employ Weighted Intra-cluster Variance (WIV) to measure the frequency-weighted deviation of experts from their cluster centroids, serving as an empirical proxy for consolidation-induced distortion. Figure~\ref{fig:wiv} shows that PRIME reduces WIV by an average of 67.21\% compared to HC-SMoE.
While this improvement is expected given our frequency-explicit objective, the magnitude of the fidelity gap exposes the critical limitation of the uniform importance assumption inherent in baselines. By solely minimizing unweighted geometric variance, frequency-agnostic baselines fail to account for the disproportionate effect of altering high-frequency experts, leading to the destructive consolidation of geometrically similar but high-impact experts.
PRIME mitigates this issue by aligning the assignment objective with usage-weighted distortion, preventing such disproportionate degradation.
Qualitatively, Figure~\ref{fig:tsne_viz} shows the structural impact of compression by projecting the expert parameter manifold using t-SNE \cite{laurens2008tsne}. Note that we use Qwen1.5-MoE here for better clarity due to its manageable number of experts, while Qwen3-MoE results are provided in the Appendix \ref{app:tsne_qwen3}.
The baseline (left) demonstrates that the greedy nature of existing heuristics can lead to locally myopic assignments. Compounded by a frequency-agnostic reliance on geometric proximity, these approaches compel the assimilation of numerous heterogeneous experts into overly dense aggregations, resulting in severe feature dilution.
In contrast, PRIME (right) maintains a spatially balanced manifold.
Additionally, red crosses represent pruned experts. PRIME removes these geometrically isolated outliers instead of forcing them into ill-fitted clusters.

\begin{figure}[t]
\centering
\centerline{\includegraphics[width=\columnwidth]{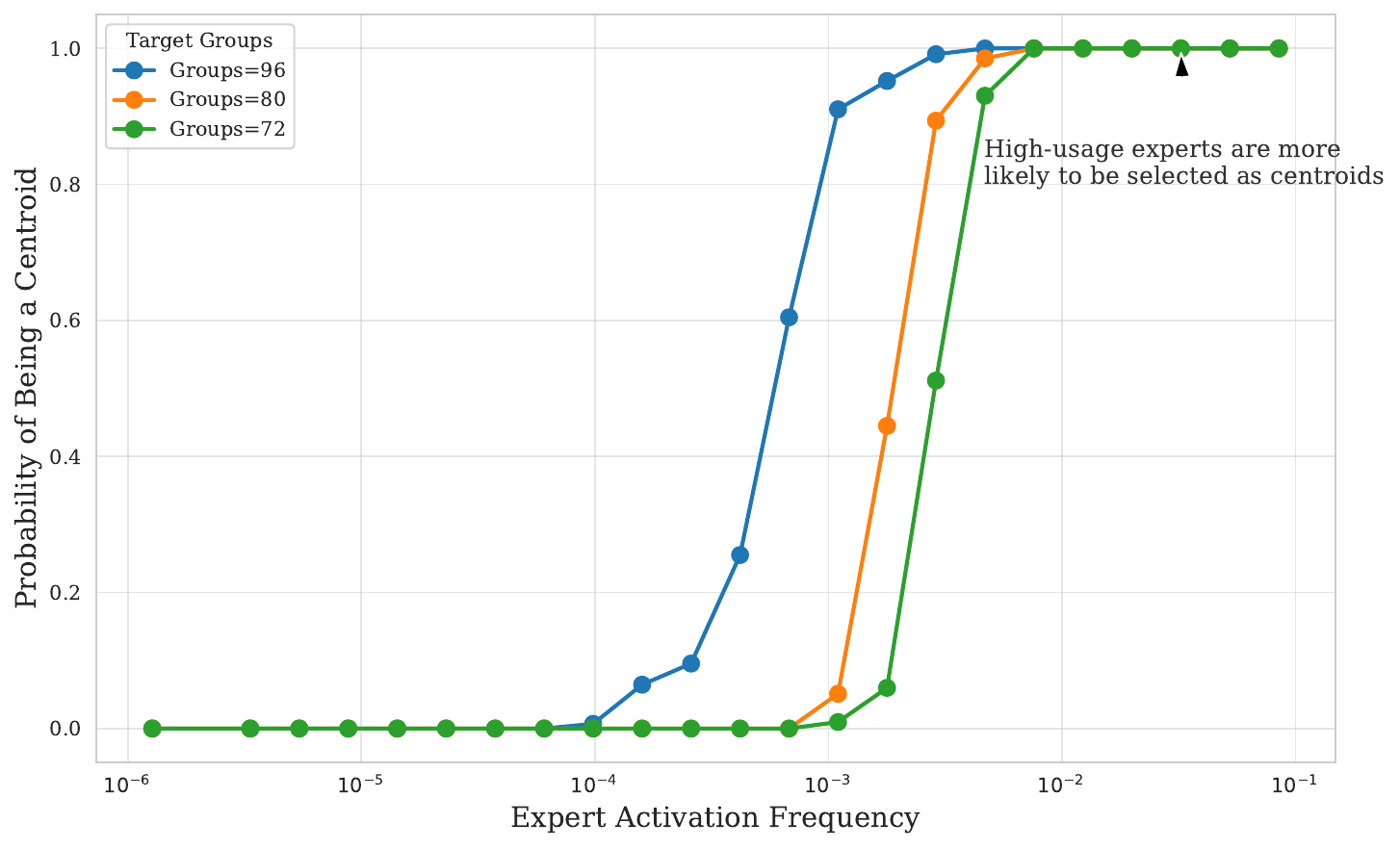}}
\caption{Centroid selection probability vs.\ expert activation frequency on Qwen3-30B-A3B across compression levels.}
\label{fig:centroid_selection}
\end{figure}

\begin{figure}[t]
\centering
\includegraphics[width=0.325\columnwidth]{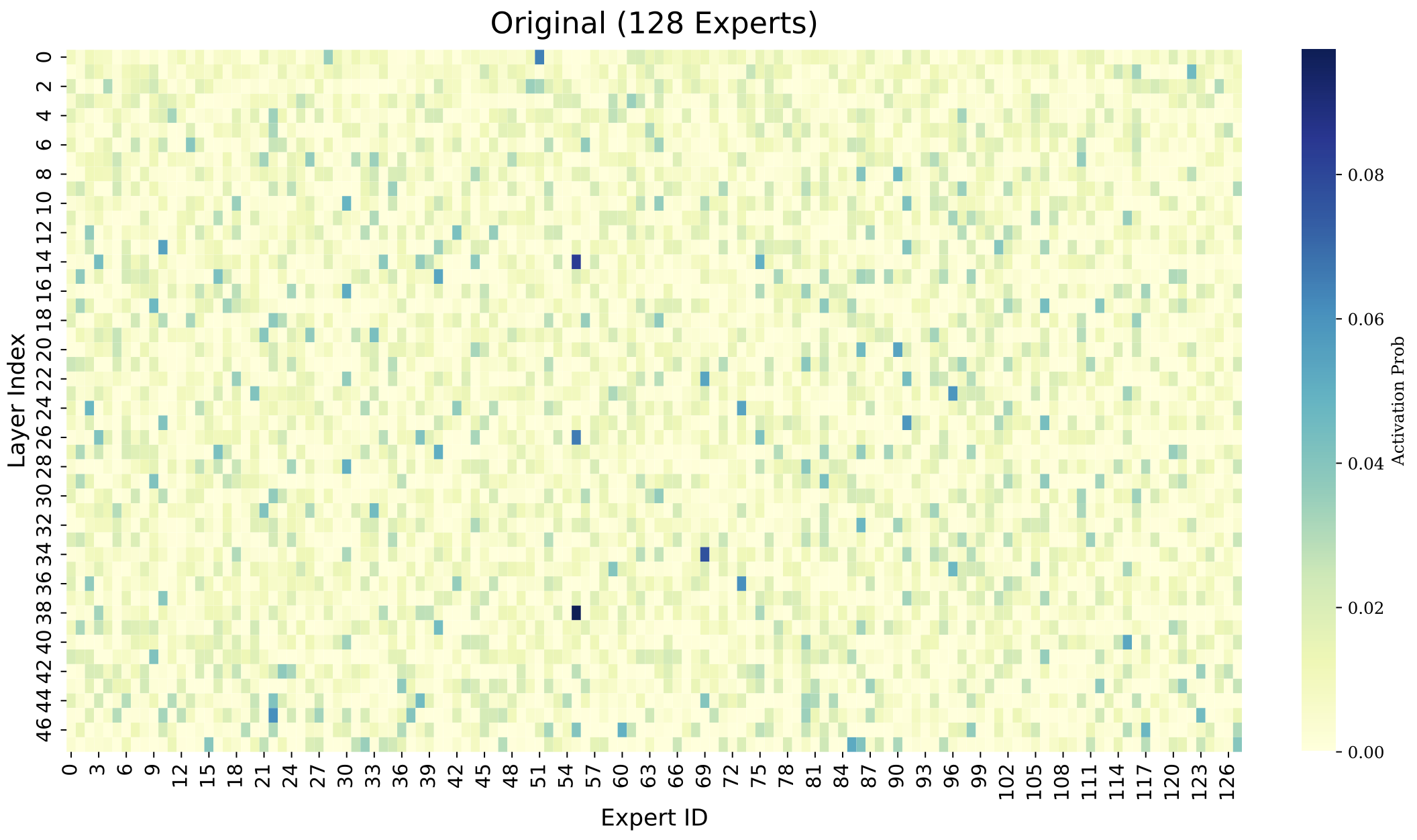}\hfill
\includegraphics[width=0.325\columnwidth]{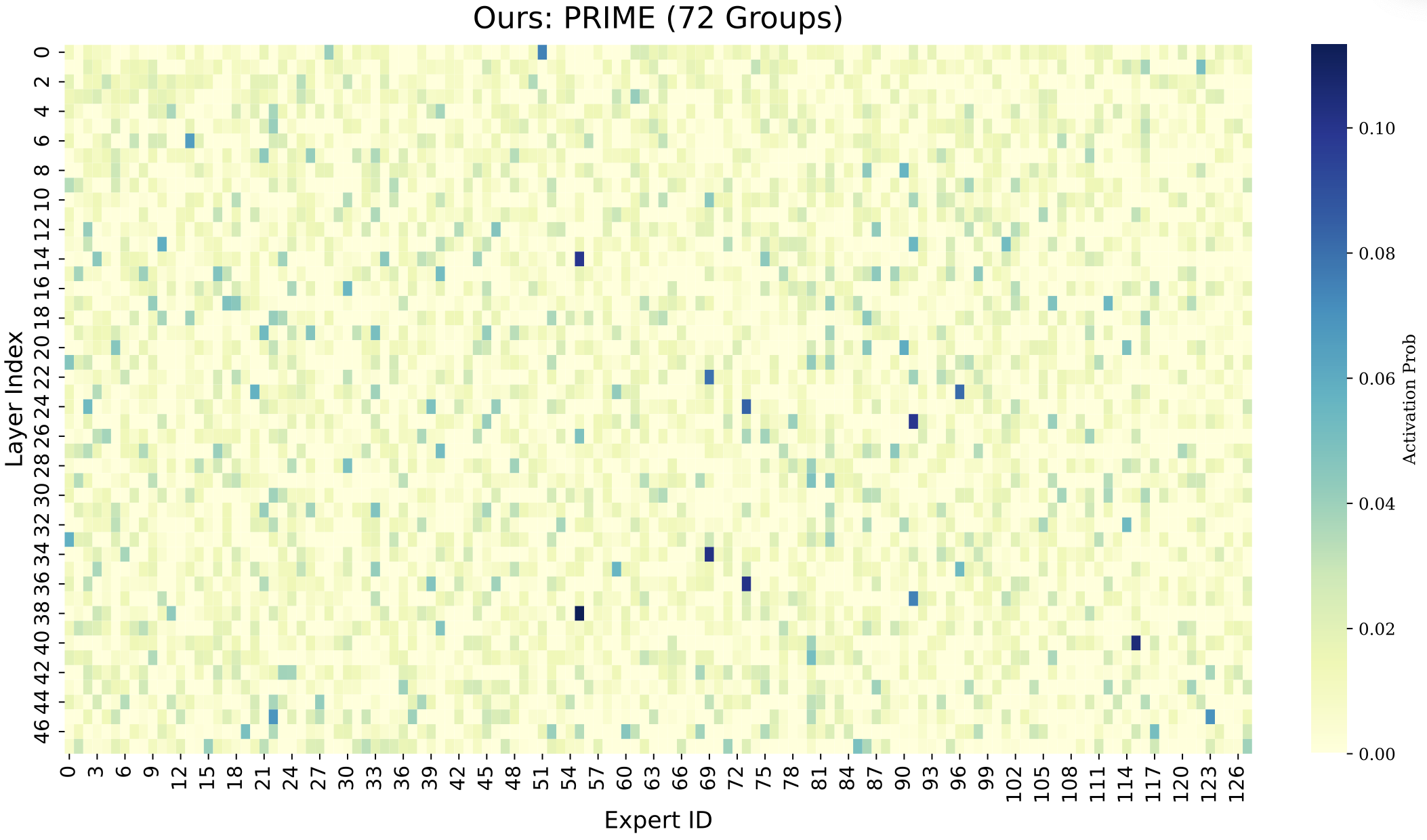}\hfill
\includegraphics[width=0.325\columnwidth]{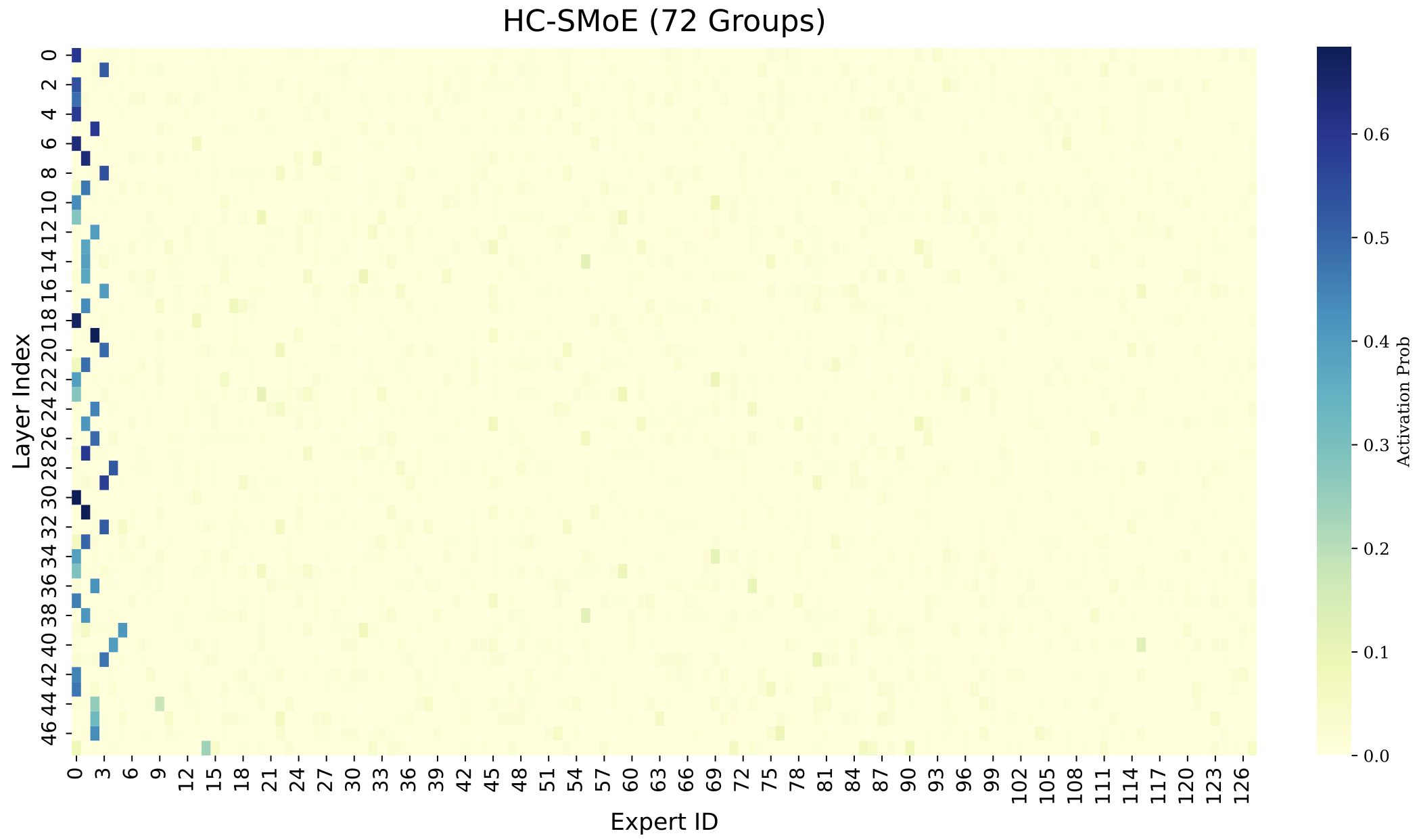}
\caption{Expert routing heatmaps on calibration set (Qwen3-30B-A3B, $K=72$). Left: Original; Middle: PRIME; Right: HC-SMoE.}
\label{fig:routing_topology}
\end{figure}

\paragraph{Implicit Soft Gating for Centroid Selection.}
We investigate the dependence of expert centroid selection on activation frequency.
Figure~\ref{fig:centroid_selection} reveals that the global budget constraint in Eq.~\ref{eq:budget} induces a sigmoid-like transition, effectively acting as an implicit soft gate.
High-frequency experts with activation frequency greater than $10^{-2}$ are systematically prioritized as centroids, consistent with the surrogate-optimality conditions derived in Appendix~\ref{app:opt_condition}.
Serving as aggregation anchors, these experts dominate the resulting parameter composition via our frequency-weighted merging in Eq.~\ref{eq:merge}, thereby reducing representational drift.
Furthermore, as the compression target tightens from $K=96$ to $K=72$, this soft gate adaptively shifts rightward.
This behavior is encouraged by the objective function in Eq.~\ref{eq:objective}, where the merging penalty scales linearly with frequency $f_i$. To reduce total weighted assignment cost under stricter constraints, the limited centroid budget is competitively reallocated to the highest-frequency nodes. Mid-frequency experts retained at $K=96$ thus fail to satisfy the centroid budget, effectively elevating the minimum activation frequency required for centroid selection.

\paragraph{Routing Stability and Load Balancing.}
Finally, we examine the impact of compression on token routing dynamics by visualizing expert activation heatmaps as shown in Figure~\ref{fig:routing_topology}. Since routing collapse is a merging-specific failure mode where pruning-based methods simply remove experts without altering the remaining routing structure, we focus the comparison on HC-SMoE.
The HC-SMoE panel on the right exhibits severe load imbalance, visually manifested as dark vertical banding. This pattern indicates that traffic is disproportionately funneled into a limited subset of over-merged centroids.
In contrast, the PRIME panel in the middle closely aligns with the routing distribution of the original model shown on the left.
Crucially, by aligning the dominant activation pathways with the uncompressed model, PRIME effectively preserves routing stability and maintains distributed expert utilization, avoiding the mode collapse observed in HC-SMoE.

\begin{table*}[t]
\caption{Scalability evaluation on Qwen3.5-122B-A10B (256 experts per layer).}
\label{tab:scalability}
\centering
\begin{small}
\begin{sc}
\resizebox{\textwidth}{!}{%
\begin{tabular}{c l ccccc c cc c c c}
\toprule
 & & \multicolumn{5}{c}{Multiple Choice} & \multicolumn{1}{c}{Long Context} & \multicolumn{3}{c}{Coding} & Math \\
\cmidrule(lr){3-7} \cmidrule(lr){8-8} \cmidrule(lr){9-11} \cmidrule(lr){12-12}
$K$ & Method & BoolQ & MMLU & OBQA & RTE & MC Avg & LongBench & HE & MBPP & Coding Avg & GSM8K \\
\midrule
256 & Original & 90.18 & 86.70 & 45.80 & 85.56 & 77.06 & 54.25 & 92.68 & 90.74 & 91.71 & 93.48 \\
\midrule
\multirow{3}{*}{192} & HC-SMoE & 86.33 & \textbf{85.53} & \textbf{46.40} & 78.70 & 74.24 & \textbf{53.88} & 80.49 & 78.57 & 79.53 & \textbf{92.80} \\
& REAP & 87.77 & 76.47 & 44.00 & 80.87 & 72.28 & 37.75 & 90.24 & \textbf{90.48} & 90.36 & 90.75 \\
& \textbf{PRIME (Ours)} & \textbf{88.78} & 84.50 & \textbf{46.40} & \textbf{84.84} & \textbf{76.13} & 53.38 & \textbf{92.07} & 89.42 & \textbf{90.75} & 92.34 \\
\midrule
\multirow{3}{*}{160} & HC-SMoE & 74.25 & \textbf{75.89} & \textbf{43.80} & 71.48 & 66.36 & 46.50 & 39.63 & 71.43 & 55.53 & 88.93 \\
& REAP & \textbf{82.91} & 66.49 & 40.00 & 79.06 & 67.12 & 35.38 & 85.37 & 82.28 & 83.83 & 88.40 \\
& \textbf{PRIME (Ours)} & 78.96 & 75.57 & 42.20 & \textbf{83.39} & \textbf{70.03} & \textbf{49.75} & \textbf{86.59} & \textbf{85.19} & \textbf{85.89} & \textbf{90.45} \\
\midrule
\multirow{3}{*}{144} & HC-SMoE & 71.19 & 68.64 & \textbf{42.60} & 70.04 & 63.12 & 45.00 & 34.76 & 59.79 & 47.28 & \textbf{85.44} \\
& REAP & \textbf{81.65} & 60.42 & 35.00 & 84.12 & 65.30 & 30.75 & \textbf{84.15} & 79.89 & 82.02 & 81.96 \\
& \textbf{PRIME (Ours)} & 69.79 & \textbf{71.85} & 41.00 & \textbf{85.56} & \textbf{67.05} & \textbf{46.63} & 82.93 & \textbf{83.07} & \textbf{83.00} & 85.06 \\
\bottomrule
\end{tabular}
}
\end{sc}
\end{small}
\end{table*}

\subsection{Scalability Analysis}
\label{sec:scalability}

To validate that PRIME scales to industrial-level MoE architectures, we evaluate on Qwen3.5-122B-A10B, a model with 256 experts per layer and 10 billion active parameters out of 122 billion total. Compared with Qwen3-30B-A3B, this model has 2$\times$ as many experts, which may pose greater challenges for the ILP solver.

\paragraph{Performance at Scale.} Table~\ref{tab:scalability} presents results across three compression levels. At $K=192$, corresponding to 25\% reduction, PRIME achieves the best MC average of 76.13 and coding average of 90.75, outperforming HC-SMoE at 74.24 / 79.53 and REAP at 72.28 / 90.36. HC-SMoE retains the best LongBench at 53.88 and GSM8K at 92.80, indicating that its clustering-based merging preserves certain capabilities at mild compression; however, its coding performance degrades significantly to 79.53 vs.\ 90.75 for PRIME. At $K=160$, corresponding to 37.5\% reduction, PRIME leads on MC with 70.03 vs.\ 67.12 for REAP and 66.36 for HC-SMoE, on coding with 85.89 vs.\ 83.83 and 55.53, and on long-context with 49.75 vs.\ 35.38 and 46.50. HC-SMoE suffers catastrophic coding collapse to a 55.53 average, driven by HumanEval dropping to 39.63. Even at $K=144$, corresponding to 43.75\% reduction, PRIME maintains the best MC average of 67.05, coding average of 83.00, and LongBench of 46.63, confirming that the ILP formulation scales effectively to 256-expert industrial models.

\paragraph{Computational Overhead Profiling.}\hspace{0.5em} A key concern for scaling ILP-based compression is computational tractability. Table~\ref{tab:overhead} profiles the end-to-end runtime of PRIME on both models. At the larger 256-expert scale, the total profiling overhead remains practical, requiring 15.5 minutes compared with 3.5 minutes for the 128-expert setting. The route collection scales from 62s to 295s and frequency statistics from 66s to 303s, both proportional to model size. The ILP solving phase grows from 42s at 0.87s per layer to 288s at 6.0s per layer, reflecting the expected increase in constraint complexity with more experts. Critically, the similarity computation scales from 38s to 46s, remaining below 1s per layer, indicating that our pairwise distance formulation does not become a bottleneck. The merging step remains negligible at under 1s at both scales. These results demonstrate that PRIME remains practically deployable even for models with 256 experts, requiring only 15.5 minutes of one-time compute that is amortized over the lifetime of the compressed model.

\begin{table}[t]
\caption{End-to-end overhead profiling of PRIME on A800 GPUs.}
\label{tab:overhead}
\centering
\begin{small}
\resizebox{\columnwidth}{!}{%
\begin{tabular}{l c ccccc c}
\toprule
{\sc Model} & {\sc Experts} & {\sc Route} & {\sc Collect} & {\sc Distance} & {\sc Solve} & {\sc Merge} & {\sc Total} \\
\midrule
Qwen3-30B-A3B & 128 & 62s & 66s & 38s & 42s & 0.13s & 3.5min \\
Qwen3.5-122B-A10B & 256 & 295s & 303s & 46s & 288s & 0.72s & 15.5min \\
\bottomrule
\end{tabular}
}
\end{small}
\end{table}

\subsection{Deployment Metrics on vLLM}
\label{sec:deployment}

We benchmark the original 128-expert Qwen3-30B-A3B and its merged 96-expert checkpoint on vLLM v0.28.0, using one A800-80GB GPU (TP1), BF16, a 16,384-token context limit, and 90\% GPU memory utilization. Both checkpoints are served directly with top-$k=8$ unchanged; measurements focus on serving efficiency, and task accuracy is not evaluated in these runs.

\begin{table}[t]
\caption{TP1 results for the merged 96-expert model: 512 input tokens, 4,096 output tokens, concurrency 128. Throughput is relative to the original 128-expert model.}
\label{tab:deployment}
\centering
\begin{small}
\resizebox{\columnwidth}{!}{%
\begin{tabular}{l cc}
\toprule
{\sc Metric} & {\sc 128 Experts} & {\sc 96 Experts} \\
\midrule
Weight Memory & 56.93 GiB & 43.43 GiB \\
KV Cache Memory & 11.17 GiB & 24.57 GiB \\
KV Capacity (tokens) & 121,968 & 268,400 \\
Output Throughput & 1.00 & 1.74 \\
\bottomrule
\end{tabular}
}
\end{small}
\end{table}

Table~\ref{tab:deployment} illustrates a memory-constrained, long-output workload. Merging from 128 to 96 experts frees 13.50 GiB of weight memory for vLLM's PagedAttention KV cache \citep{kwon2023vllm}. Capacity grows from 121,968 to 268,400 tokens, and output throughput improves by 1.74$\times$. KV storage per token is unchanged; more tokens can remain resident under the same memory budget.

The gains reflect both reduced expert-weight traffic and expanded KV capacity, rather than KV space alone. Single-request and prefill-dominated workloads show little or no speedup. Appendix~\ref{app:serving_benefits} gives the full analysis, including controlled comparisons, setup details, and measurement limits. These measured benefits apply to the merged checkpoints under the stated hardware and workload conditions.

\section{Conclusion}
\label{sec:conclusion}
We present PRIME, a unified framework that formulates MoE compression as a frequency-aware surrogate optimization problem over expert pruning and merging assignments. PRIME departs from prior disjoint pipelines and enables coordinated decisions that better preserve model capacity under parameter reduction. Incorporating expert usage frequency into the assignment objective reduces compression-induced drift and helps maintain routing stability, making PRIME a frequency-aware and empirically effective approach to MoE compression.

\clearpage
\section*{Limitations}
Our work has several limitations. First, PRIME optimizes each layer independently, without modeling cross-layer dependencies in expert utilization. A globally coordinated multi-layer formulation could potentially yield better compression decisions. Second, although we use a mixed calibration corpus (C4, Evol-CodeAlpaca, NuminaMath) to improve coverage, the optimal calibration mixture may depend on the target deployment domain, and we leave a systematic study of dataset mixing ratios to future work. Third, our evaluation focuses on standard English-centric benchmarks and may not capture performance variation across multilingual, safety-critical, or interactive settings. As a result, compressed models should be re-evaluated before deployment in specialized applications. Finally, PRIME is a one-shot post-training method without fine-tuning recovery, which may limit achievable compression ratios compared to methods that allow retraining.

\section*{Ethical Considerations}
This work focuses on structural compression of pre-trained MoE models via mathematical optimization. It does not involve human subjects, sensitive data, or generation of harmful content. The calibration and evaluation datasets used in this study are publicly available resources; users of PRIME should comply with the licenses and terms of the original models and datasets, and any released compressed checkpoints should preserve the attribution, use restrictions, and redistribution requirements of the corresponding source checkpoints. The compressed models inherit the biases and limitations of the original pre-trained models; practitioners should apply standard safety evaluations before deployment. Our method reduces computational costs, which may contribute to more equitable access to large language models.

\begingroup
\setlength{\emergencystretch}{3em}
\tolerance=2000
\bibliographystyle{icml2026}
\bibliography{references}
\endgroup

\clearpage
\onecolumn
\appendix


\section{Theoretical Analysis}
\label{app:theory}

In this section, we present theoretical analysis to motivate the design choices of PRIME. We show that a frequency-weighted objective yields a tighter upper bound on inference error than frequency-agnostic alternatives (justifying the role of $f_i$), and derive a sufficient condition under which high-frequency experts must be selected as centroids in the optimal solution. We note that the bound is stated in parameter space; the connection to the output-space metric used in practice is discussed in Remark~\ref{rem:metric_gap}.

\subsection{Motivating Frequency Weighting via an Error Bound}
\label{app:error_bound}
We provide a theoretical motivation for the frequency-weighted structure of our objective by showing that minimizing the frequency-weighted parameter distance minimizes an upper bound on the expected layer-wise inference error. We emphasize that this analysis justifies the weighting scheme, not the specific distance metric used in practice (see Remark~\ref{rem:metric_gap}).
\paragraph{Notations}
\begin{itemize}
    \item Let $x \in \mathbb{R}^d$ be the input vector drawn from the data distribution $\mathcal{D}$.
    \item The MoE layer output is formulated as $y(x) = \sum_{i=1}^N g_i(x) E(x; w_i)$, where $N$ denotes the total number of experts. $g_i(x) \in [0, 1]$ represents the gating probability for the $i$-th expert, satisfying the simplex constraint $\sum_{i=1}^N g_i(x) = 1$. $E(\cdot; w_i)$ is the function of the $i$-th expert parameterized by weights $w_i$.
    \item The compressed output is defined as $\hat{y}(x) = \sum_{i=1}^N g_i(x) E(x; w_{\pi(i)})$, where the assignment function $\pi: \{1, \dots, N\} \to \{1, \dots, N\}$ maps each expert $i$ to its assigned cluster, and $w_{\pi(i)}$ denotes the aggregated parameters of the merged experts within that cluster.
    \item Let $f_i = \mathbb{E}_{x \sim \mathcal{D}}[g_i(x)]$ denote the global activation frequency of the $i$-th expert over the distribution $\mathcal{D}$.
\end{itemize}

\textbf{Assumption 1 (Parameter-Space Lipschitz Continuity).}
We assume the expert function $E(x; w)$ is $L$-Lipschitz continuous with respect to its parameters $w$, uniformly over bounded input and parameter domains. That is, there exists a constant $L > 0$ such that for any input $x \in \mathcal{X}$ and parameter configurations $w, w' \in \mathcal{W}$:
\begin{equation}
    \|E(x; w) - E(x; w')\|_2 \le L \|w - w'\|_2.
\end{equation}
\textit{Justification:} This assumption is grounded in the first-order Taylor expansion of the neural network function. Analyzing the sensitivity of network outputs to parameter perturbations via a linear bound, i.e. Lipschitz condition is a standard practice in generalization theory~\citep{neyshabur2017exploring, jiang2019fantastic} and model compression~\citep{lecun1989optimal}. 

\begin{theorem}[Frequency-Weighted Error Bound]\label{th:error_bound}
Under Assumption 1, the expected inference error $\mathcal{L}_{err}$ is upper-bounded by the frequency-weighted parameter distance:
\[
\mathcal{L}_{err} \le L \sum_{i=1}^N f_i \|w_i - w_{\pi(i)}\|_2.
\]
Consequently, minimizing the frequency-weighted objective minimizes this upper bound. When activation frequencies are non-uniform, this bound is tighter than the frequency-agnostic bound $L \sum_i \|w_i - w_{\pi(i)}\|_2$.
\end{theorem}

\begin{proof}
We define the expected error as the expected $L_2$ norm of the output difference:
\begin{equation}
    \mathcal{L}_{err} = \mathbb{E}_{x} \left[ \| y(x) - \hat{y}(x) \|_2 \right]
\end{equation}
Substituting the MoE definition and applying the Triangle Inequality:
\begin{align}
    \mathcal{L}_{err} &= \mathbb{E}_{x} \left[ \left\| \sum_{i=1}^N g_i(x) (E(x; w_i) - E(x; w_{\pi(i)})) \right\|_2 \right] \\
    &\le \sum_{i=1}^N \mathbb{E}_{x} \left[ g_i(x) \| E(x; w_i) - E(x; w_{\pi(i)}) \|_2 \right]
\end{align}
Applying Assumption 1, the output difference term is bounded by $L\|w_i - w_{\pi(i)}\|_2$. Since this bound is uniform over $x$, it can be factored out of the expectation:
\begin{align}
    \mathcal{L}_{err} &\le \sum_{i=1}^N \mathbb{E}_{x} \left[ g_i(x) \cdot L \|w_i - w_{\pi(i)}\|_2 \right] \\
    &= L \sum_{i=1}^N \|w_i - w_{\pi(i)}\|_2 \cdot \underbrace{\mathbb{E}_{x}[g_i(x)]}_{\text{definition of } f_i}
\end{align}
Thus, we arrive at the final bound:
\begin{equation}
    \mathcal{L}_{err} \le L \sum_{i=1}^N f_i \|w_i - w_{\pi(i)}\|_2
\end{equation}
For the tightness claim, note that since $f_i \le 1$ and $\sum_i f_i =1$, the frequency-agnostic bound $L\sum_i \|w_i - w_{\pi(i)}\|_2$ is always valid but potentially looser. When frequencies are non-uniform (as empirically observed in MoE routing), the weighted bound is strictly tighter for any assignment $\pi$.
\end{proof}

\paragraph{Key Insight: Frequency Provides a Tighter Bound.}
The derived bound $L \sum f_i \|\Delta w\|_2$ provides theoretical justification for our objective function. Frequency-agnostic baselines minimize the unweighted distance, which corresponds to the special case of uniform activation $f_i = 1/N$.
Since MoE activation patterns are empirically highly skewed (see Appendix~\ref{app:expert_stats}), the uniform assumption is a poor fit. By explicitly incorporating $f_i$, our objective focuses the compression budget on high-frequency experts, yielding a tighter upper bound on the actual inference error when activation frequencies are non-uniform.

\begin{remark}[Gap Between Theoretical Metric and Practical Surrogate]\label{rem:metric_gap}
Theorem~\ref{th:error_bound} is stated in terms of parameter-space distance $\|w_i-w_{\pi(i)}\|_2$, whereas the ILP in Section~\ref{sec:ilp_formulation} uses $D_{ij}$, the normalized Euclidean distance between mean expert outputs on the calibration set. These are not the same quantity. The theorem's role is to justify the frequency-weighted structure of the objective---i.e., that distortions in high-frequency experts contribute proportionally more to inference error---not to prescribe the specific distance function. In practice, we adopt an output-space metric because (i) it directly measures functional divergence between experts without requiring the (unknown) Lipschitz constant $L$, and (ii) it is empirically effective, following prior work \citep{chen2025hcsmoe}. By the Lipschitz assumption, a small parameter distance implies a small output distance, but the converse need not hold; output-space proximity is therefore a necessary but not sufficient condition for parameter-space proximity. Consequently, our practical surrogate may admit assignments that the parameter-space bound would penalize more heavily, meaning the ILP optimizes a related but distinct surrogate rather than directly minimizing the theoretical bound. We view this as a standard theory-to-practice gap analogous to optimizing tractable surrogates (e.g., cross-entropy) in place of intractable objectives (e.g., 0-1 loss) in machine learning.
\end{remark}


\subsection{Sufficient Condition for Centroid Selection}
\label{app:opt_condition}
In this section, we derive a structural property of the optimal solution to our ILP formulation. Our analysis reveals a frequency-dependent sufficient condition under which an expert must be selected as a centroid.

\paragraph{Problem Setup.}
We formally state the optimization problem underlying PRIME. Given $N$ experts with pairwise distances $D_{ij} \ge 0$ (where $D_{ii} = 0$), activation frequencies $f_i > 0$, and a pruning penalty $\lambda > 0$, the ILP selects a set $S$ of exactly $K$ centroids to minimize:
\begin{equation}
\label{eq:ilp_formal}
C(S) = \sum_{i=1}^N f_i \cdot \min\!\left(\min_{c \in S} D_{ic},\; \lambda\right), \quad \text{s.t. } |S| = K.
\end{equation}
Each non-centroid expert $E_i$ is either merged to its nearest centroid (cost $f_i D_{ic}$) or pruned (cost $f_i \lambda$), whichever is cheaper.

\begin{theorem}[Sufficient Condition for Centroid Selection]\label{th:sufficient}
Let $S^*$ be the set of $K$ centroids in an optimal solution. For each centroid
$E_v \in S^*$, let $G_v$ denote its cluster, and define the dissolution cost
increment of $G_v$ as
\[
\Delta_v
:= \sum_{E_i \in G_v} f_i \left(
   \min\Bigl(\underbrace{\min_{c' \in S^* \setminus \{E_v\}} D_{ic'}}_{Reassignment},\; \underbrace{\lambda}_{Pruning} \Bigr)
   \;-\; \underbrace{D_{iv}}_{Original\; Cost} \right),
\]
and define the Minimum Cluster Dissolution Cost as
\[
\Delta_{\mathrm{min}} := \min_{E_v \in S^*} \Delta_v.
\]
If an expert $E_k$ satisfies

\begin{equation}
\label{eq:freq_anchor}
 f_k \;>\;
\frac{\Delta_{\mathrm{min}}}{\min\bigl(\min_{c \in S^*} D_{kc},\; \lambda\bigr)},
\end{equation}

 $E_k$ is guaranteed to be a centroid in the optimal solution
(i.e.\ $z_k = 1$).
\end{theorem}

\begin{proof}
Assume for contradiction that $E_k \notin S^*$ in an optimal solution $S^*$ with objective value $C(S^*)$.

Since $E_k \notin S^*$, its cost contribution in the optimal solution is
\begin{equation}
\mathrm{Cost}(E_k)
= f_k \cdot \min\bigl(\min_{c \in S^*} D_{kc},\; \lambda\bigr).
\label{eq:Ek-cost}
\end{equation}
We construct an alternative solution by (i) adding $E_k$ as a centroid, yielding $K+1$ centroids, then (ii) removing the centroid $E_{v^*}$ whose cluster dissolution incurs the least cost increase:
\[
v^* \in \arg\min_{E_v \in S^*} \Delta_v.
\]
Let $S' = (S^* \cup \{E_k\}) \setminus \{E_{v^*}\}$, so $|S'| = K$.

\textit{Cost reduction from adding $E_k$:} Making $E_k$ a centroid (with self-distance 0) eliminates its own cost, providing a reduction of at least:
\begin{equation}
\mathrm{Gain}
\ge f_k \cdot \min\bigl(\min_{c \in S^*} D_{kc},\; \lambda\bigr).
\label{eq:gain}
\end{equation}
(The gain may be larger if other experts also reassign to $E_k$, but we use this lower bound.)

\textit{Cost increase from dissolving $E_{v^*}$:} Each member of $G_{v^*}$ must reassign to its next-best centroid in $S' \setminus \{E_{v^*}\} = (S^* \setminus \{E_{v^*}\}) \cup \{E_k\}$, or be pruned. The cost increase is at most $\Delta_{v^*}$, because the availability of $E_k$ as an additional reassignment target can only reduce costs relative to the definition of $\Delta_v$ (which excludes $E_k$). Thus:
\begin{equation}
\mathrm{Penalty} \le \Delta_{v^*} = \Delta_{\mathrm{min}}.
\label{eq:penalty}
\end{equation}

Combining:
\[
C(S') - C(S^*) \le \mathrm{Penalty} - \mathrm{Gain} \le \Delta_{\mathrm{min}} - f_k \cdot \min\bigl(\min_{c \in S^*} D_{kc},\; \lambda\bigr).
\]
The condition \eqref{eq:freq_anchor} ensures the right-hand side is negative, hence $C(S') < C(S^*)$, contradicting the optimality of $S^*$.
\end{proof}

\paragraph{Interpretation: High-Frequency Experts as Natural Anchors.}
Theorem \ref{th:sufficient} provides a sufficient condition under which high-frequency experts are guaranteed to be centroids. The condition~\eqref{eq:freq_anchor} shows that as an expert's activation frequency $f_k$ grows relative to the minimum dissolution cost, the expert must be retained as a centroid to maintain optimality.
This has a practical consequence: the ILP naturally avoids merging heavily-used experts into other clusters, since doing so would incur a large cost that could be eliminated by making such experts centroids. While this does not constitute a formal load-balancing guarantee, it provides a mechanism through which the optimization tends to preserve high-traffic routing paths, contributing to the empirically observed routing stability (Section~\ref{sec:analysis}).

\clearpage
\section{Serving Performance Benefits}
\label{app:serving_benefits}

\paragraph{Scope and controls.}
We extend the analysis in Section~\ref{sec:deployment} by comparing the original 128-expert Qwen3-30B-A3B model with its merged 96- and 64-expert variants. All serving measurements are collected by running these checkpoints directly under a common configuration: vLLM v0.28.0 on A800-80GB GPUs with bf16, top-$k=8$, 90\% GPU memory utilization, a 16,384-token context limit, up to 512 active sequences, disabled prefix caching, and fixed output lengths with EOS ignored. To distinguish the effects of memory capacity, we consider three deployment arms: TP1 with automatically allocated KV cache, TP1 with the KV capacity of every variant fixed at 121,968 tokens, and TP2 with ample KV capacity (871,584 tokens for the baseline). All ratios compare a compressed variant with the 128-expert baseline in the same arm and therefore do not compare results across different GPU counts.

\begin{center}
\begin{minipage}{\linewidth}
\captionsetup{hypcap=false}
\centering
\includegraphics[width=\linewidth]{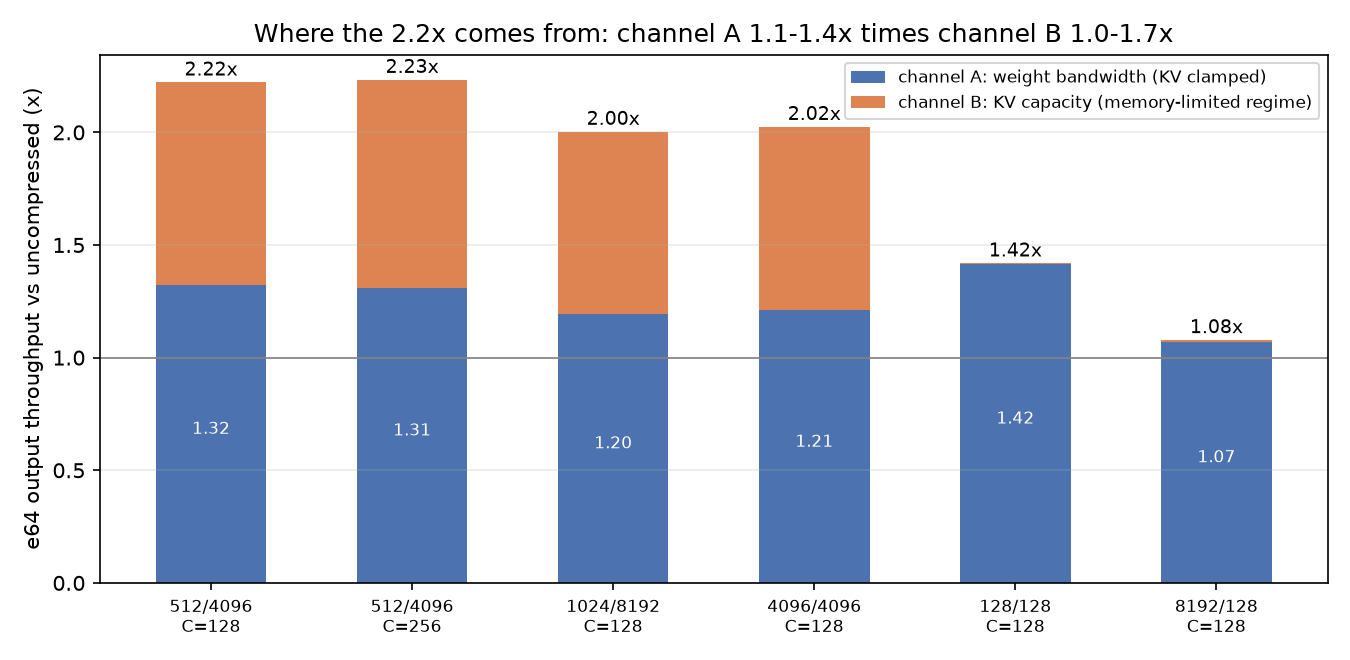}
\captionof{figure}{TP1 throughput gains of the merged 64-expert model over the original 128-expert model. Labels indicate input/output lengths and concurrency $C$. Blue bars show fixed-KV gains, whereas orange heights show the \emph{difference} between free- and fixed-KV speedups rather than the capacity multiplier itself. This decomposition is descriptive and does not establish independent causal effects.}
\label{fig:serving_attribution}
\end{minipage}
\end{center}

\paragraph{Expert-weight bandwidth.}
Although keeping eight active experts per token leaves the amount of expert computation approximately unchanged, compression allows a concurrent decode batch to touch fewer distinct experts. Under independent uniform routing, the expected number of experts touched is $E_{\mathrm{touch}}(B)=E[1-(1-8/E)^B]$, where $B$ denotes the actual decode batch size rather than the offered request concurrency. The resulting reduction in weight traffic is consistent with the controlled gains in Figures~\ref{fig:serving_attribution}--\ref{fig:serving_concurrency}. However, this estimate is only a heuristic and should not be interpreted as a physical speedup bound because routing skew, kernel efficiency, attention, and communication also affect runtime. For the 128/128 workload at $C=128$, the free- and fixed-KV gains both equal 1.42$\times$. In this case, 128 full sequences require only 32,768 KV tokens, which remains below the baseline capacity.

\paragraph{KV-cache capacity.}
On TP1, weight memory decreases from 56.93 GiB for 128 experts to 43.43 GiB for 96 experts and 29.92 GiB for 64 experts. This reduction increases the corresponding KV-cache capacity from 121,968 tokens to 268,400 and 417,008 tokens, representing gains of 2.20$\times$ and 3.42$\times$, while KV storage remains 96 KiB per token. For workloads with long outputs, the additional capacity accommodates more concurrent tokens and thereby reduces preemption. At 512/4096 and concurrency 128, for example, the gain from the 64-expert model increases from 1.32$\times$ with fixed KV to 2.22$\times$ with free KV. The ratio between these gains is 1.68, which characterizes the additional capacity-associated benefit but does not establish independent multiplicative causes. When concurrency increases to 256, throughput reaches 2.23$\times$ and the number of preemptions falls from 693 to 264. These results also show that capacity planning must account for both input and generated tokens rather than input length alone.

\newpage
\begin{center}
\begin{minipage}{\linewidth}
\captionsetup{hypcap=false}
\centering
\includegraphics[width=\linewidth]{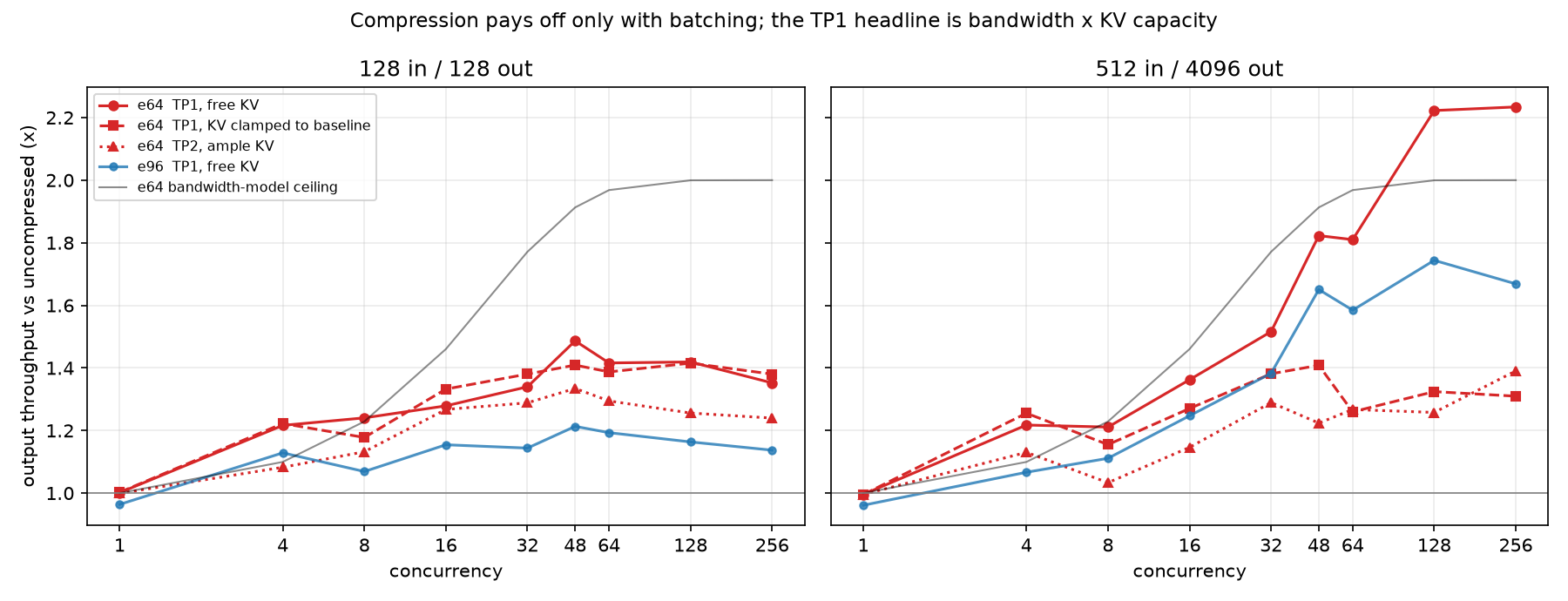}
\captionof{figure}{Serving gains versus concurrency for short outputs (left) and long outputs (right). ``e64/e96'' denotes the merged 64-/96-expert models, and each curve is normalized within its deployment arm. As KV pressure increases, the long-output TP1 curve separates from the controls. The gray ``ceiling'' represents a uniform-routing bandwidth estimate rather than a rigorous speedup bound.}
\label{fig:serving_concurrency}
\end{minipage}
\end{center}

\paragraph{Deployment implications and limits.}
As shown in Figure~\ref{fig:serving_concurrency}, the largest gains occur for concurrent, memory-constrained workloads with long outputs. At 512/4096 and $C=128$, the 64-expert model achieves a 2.22$\times$ gain on TP1, compared with 1.32$\times$ under fixed KV capacity and 1.26$\times$ on TP2. Under the same TP1 setting, the 96-expert model achieves a 1.74$\times$ gain. By contrast, single-request decoding shows no meaningful improvement for the 64-expert model and a 3--4\% slowdown for the 96-expert model, while performance for prefill-dominated serving remains near unity. Throughput under admission pressure should also be considered alongside time to first token because a low time per output token may simply reflect fewer admitted requests. Selected repeat runs show throughput variations of approximately 1.5\% for the baseline and 64-expert models and 5\% for the 96-expert model. Overall, the attainable speedup depends strongly on the workload and the degree of memory pressure.

\section{Empirical Analysis of Expert Utilization}
\label{app:expert_stats}

We analyze the expert routing statistics of Qwen3-30B-A3B on the calibration dataset. Figure \ref{fig:expert_dist} presents the activation frequency distribution and its corresponding Cumulative Distribution Function (CDF).

\begin{figure}[htbp]
    \centering
    \includegraphics[width=0.8\linewidth]{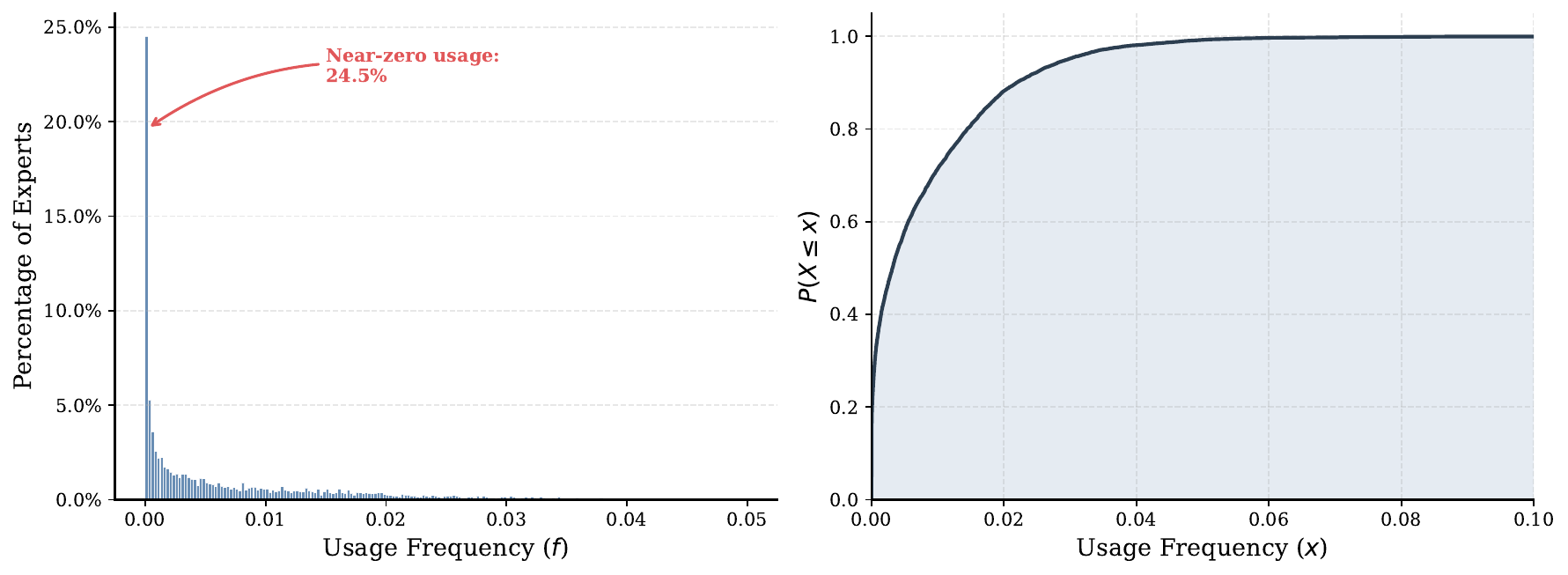} 
    \caption{Expert activation frequency distribution on Qwen3-30B-A3B. Left: histogram; Right: CDF.}
    \label{fig:expert_dist}
\end{figure}

The histogram in Figure \ref{fig:expert_dist} (Left) illustrates a pronounced long-tail distribution of expert utilization. A significant portion of the expert population is rarely activated, with statistics indicating that 24.5\% of experts fall into the ``Near-zero usage'' category. 
This observation provides a critical insight into the compression dynamics observed in our experiments, particularly explaining why the performance gap between pruning and merging diminishes at high compression rates. According to our frequency-weighted merging formulation (Eq.~\ref{eq:merge}), when an expert $E_i$ has a negligible activation frequency ($f_i \approx 0$), its contribution to the aggregated centroid weights is mathematically insignificant. Consequently, merging such a low-frequency expert into a high-frequency centroid results in a parameter composition that is virtually identical to the centroid itself. In this regime, merging effectively degenerates into pruning, as the low-frequency expert's information is diluted to the point of vanishing during the weighted aggregation.
The CDF curve in Figure \ref{fig:expert_dist} (Right) further quantifies this activation sparsity. The steep initial rise indicates that a vast number of experts share this low-utility characteristic. PRIME's frequency-aware objective (Eq.~\ref{eq:objective}) naturally exploits this distribution by assigning minimal weight to these near-zero experts. This design enables the ILP solver to flexibly process them via pruning or merging with minimal penalty, ensuring that the model's primary inference pathways remain uncompromised.

\section{Impact of Penalty Threshold}
\label{app:sensitivity}

We conduct experiments on Qwen3-30B-A3B model to investigate the impact of penalty threshold. Theoretically, the penalty threshold $\lambda$ governs the trade-off between expert pruning and merging by controlling the cost of pruning. Smaller values reduce the penalty for expert removal and thus favor pruning when similarity is low, whereas larger values discourage pruning and favor expert merging under the global parameter constraint.

\begin{figure}[htbp]
    \centering
    \includegraphics[width=0.7\linewidth]{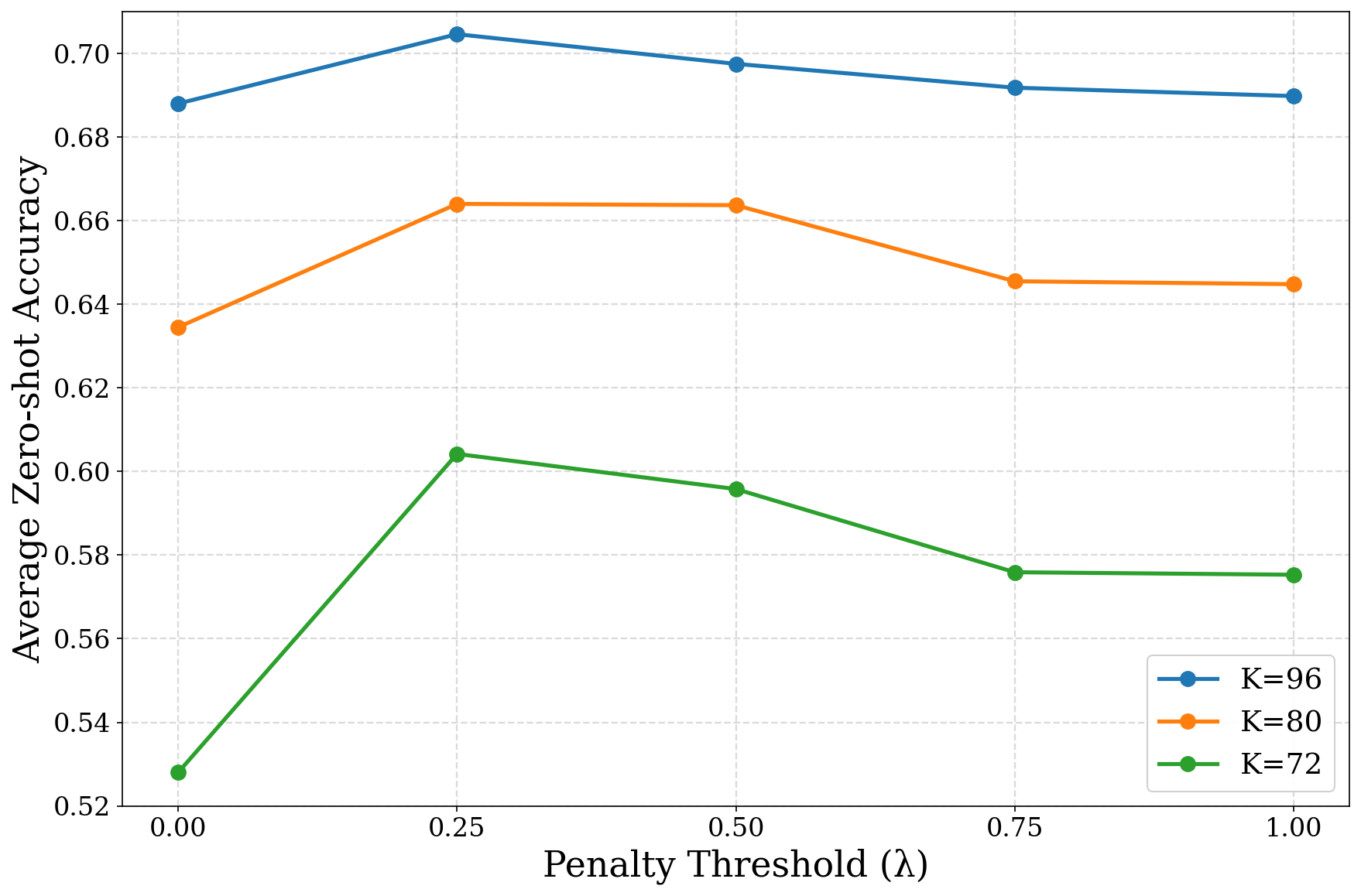}
    \caption{Impact of penalty threshold $\lambda$ on zero-shot performance across compression levels.}
    \label{fig:lambda_sensitivity}
\end{figure}

Figure \ref{fig:lambda_sensitivity} shows the effect of $\lambda$ on zero-shot accuracy with selected MC tasks under different sparsity levels ($K \in \{96, 80, 72\}$). We use this subset as a lightweight proxy for hyperparameter selection, as exhaustively sweeping $\lambda$ across all evaluation dimensions would be prohibitively expensive. Extreme settings consistently lead to inferior generalization on zero-shot tasks. When $\lambda=0$, the method collapses to prune-only and suffers a sharp performance drop, particularly under strict constraints, indicating that aggressive pruning removes task-critical experts in addition to redundant ones. Conversely, as $\lambda$ approaches $1$, the solver is forced into a merge-only regime, where experts are merged regardless of geometric compatibility. Although more parameters are retained, this induces feature interference by collapsing functionally dissimilar experts, leading to degraded representations and reduced accuracy.

Across all sparsity levels, performance peaks at $\lambda=0.25$, indicating that moderate pruning combined with selective merging yields the best trade-off. This configuration effectively filters structural outliers to mitigate feature interference, striking a robust balance between information retention and representational coherence. Consequently, we adopt $\lambda=0.25$ as the default setting.

\section{t-SNE visualization of Qwen3-30B-A3B}
\label{app:tsne_qwen3}

Figure~\ref{fig:tsne_qwen3} presents the t-SNE projection of the expert parameter manifold for Qwen3-30B-A3B. In the main text, we prioritized the visualization of Qwen1.5-MoE to ensure visual clarity, as its smaller expert population allows for a more legible depiction of the manifold structure. However, given that Qwen3-30B-A3B serves as the primary backbone for our analysis, we provide its corresponding visualization here to verify the scalability and consistency of our observations.

Despite the significantly increased density of experts in Qwen3, the structural behaviors remain consistent with those observed in the smaller model. The baseline method HC-SMoE (left) continues to exhibit severe clustering issues, where greedy frequency-agnostic heuristics force a large number of heterogeneous experts into a few massive, over-concentrated aggregations. This excessive merging results in feature dilution, as the distinct representations of diverse experts are collapsed into a single centroid. In contrast, PRIME (right) successfully maintains a spatially balanced and well-distributed manifold even under this higher expert density. Furthermore, the red crosses demonstrate the robustness of our unified pruning-and-merging framework which identifies geometrically isolated experts that are unsuitable for merging and strategically prunes them, avoiding forcing outliers into ill-fitted clusters as observed in baselines.

\begin{figure}[htbp]
\centering
\includegraphics[width=0.7\linewidth]{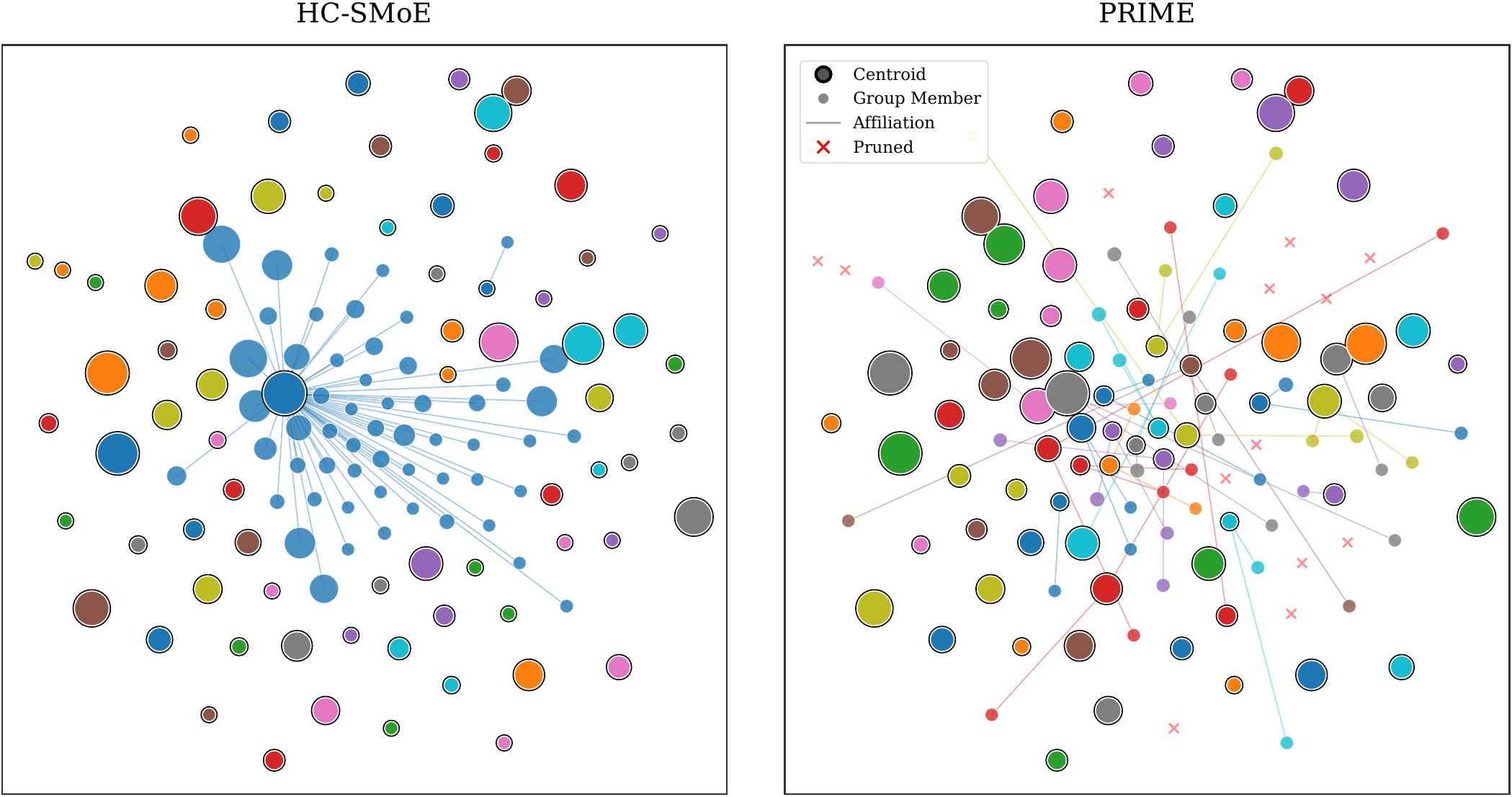} 
\caption[t-SNE visualization of Qwen3-30B-A3B.]{%
t-SNE visualization of Qwen3-30B-A3B.}
\label{fig:tsne_qwen3}
\end{figure}

\section{Routing Stability on Downstream Reasoning Tasks}
\label{app:routing_obqa}
\begin{figure}[htbp]
\centering
\begin{subfigure}[t]{0.32\textwidth}
\centering
\includegraphics[width=\linewidth]{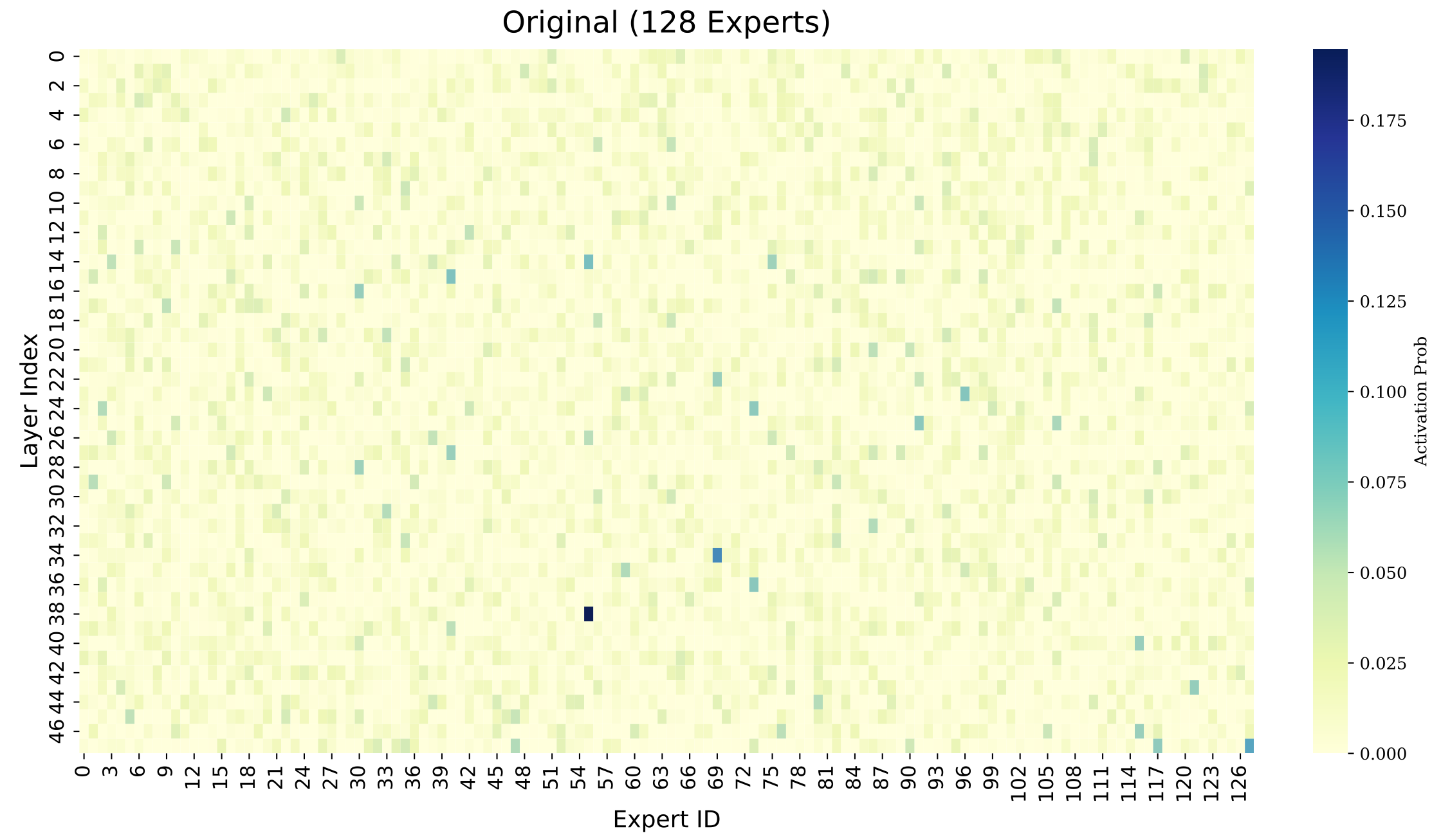}
\caption{Original}
\end{subfigure}
\hfill
\begin{subfigure}[t]{0.32\textwidth}
\centering
\includegraphics[width=\linewidth]{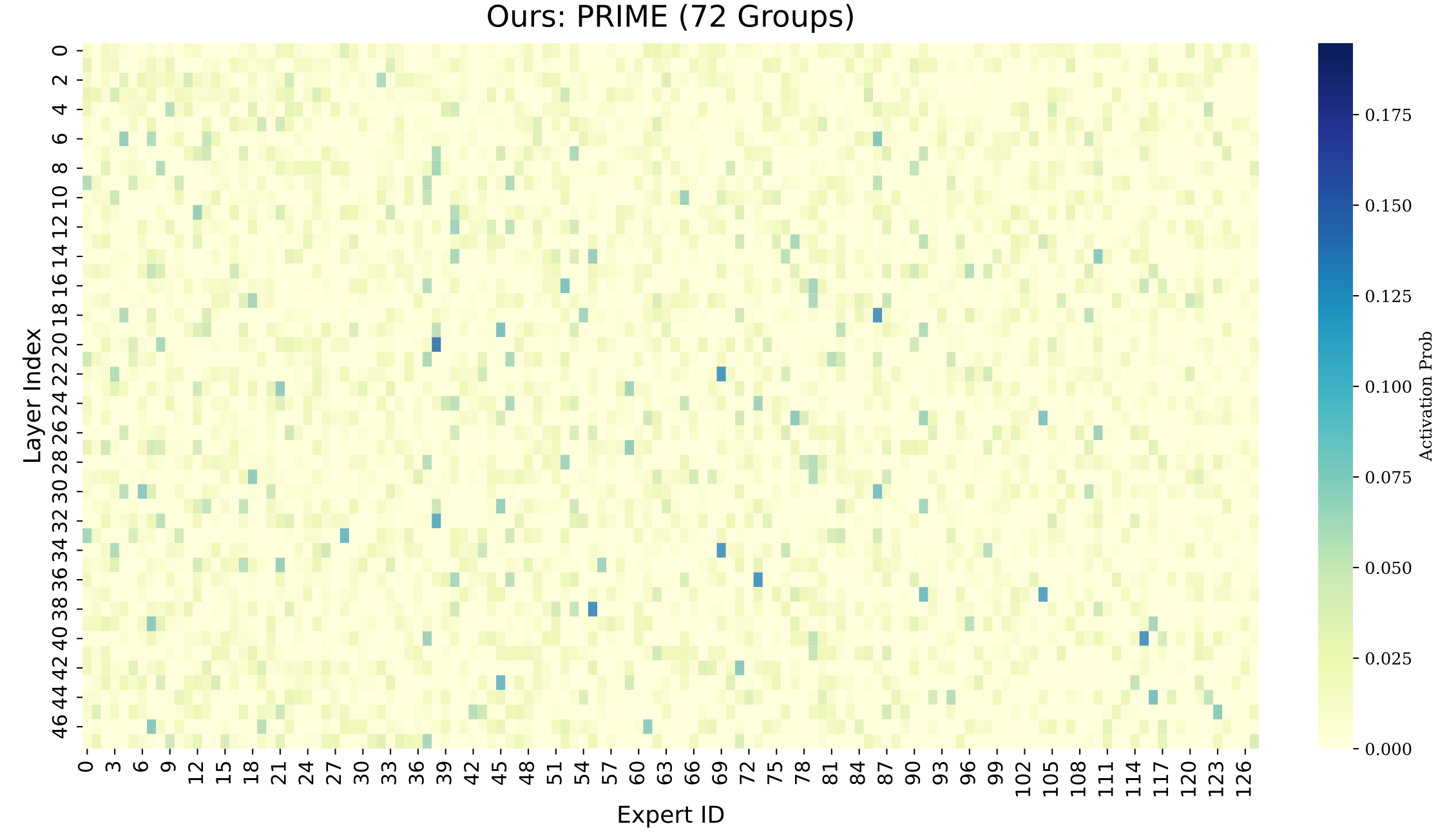}
\caption{PRIME}
\end{subfigure}
\hfill
\begin{subfigure}[t]{0.32\textwidth}
\centering
\includegraphics[width=\linewidth]{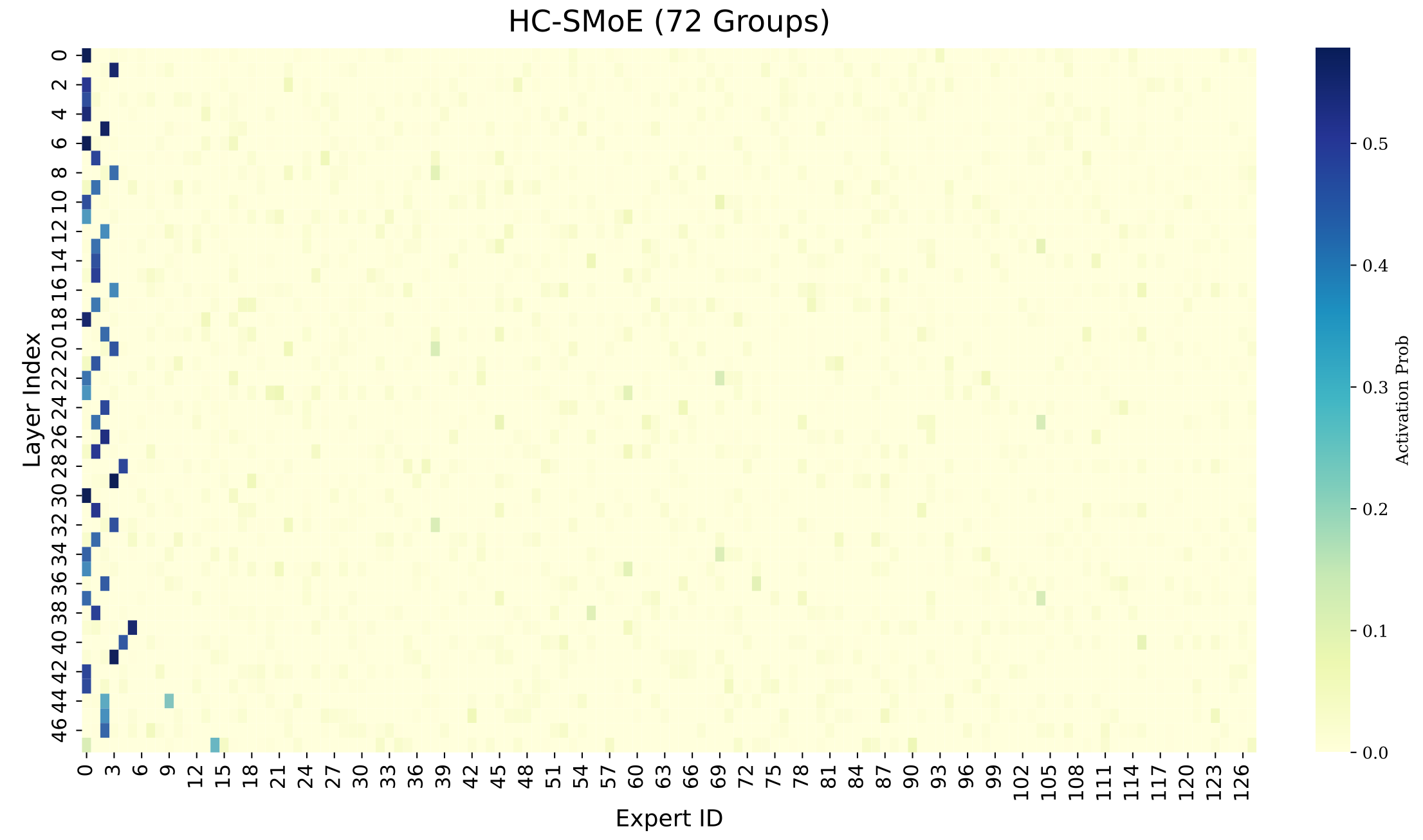}
\caption{HC-SMoE}
\end{subfigure}
\caption{Expert routing heatmaps on OpenBookQA (Qwen3-30B-A3B). Left: Original model; Middle: PRIME; Right: HC-SMoE.}
\label{fig:routing_topology_obqa}
\end{figure}

Our calibration and initial routing heatmap analysis in Section~\ref{sec:analysis} are both based on the calibration dataset, thus it is necessary to verify whether PRIME can generalize to downstream tasks rather than overfitting to the calibration statistics. Figure~\ref{fig:routing_topology_obqa} visualizes expert activation heatmaps on the OpenBookQA benchmark using Qwen3-30B-A3B. Overall, we observe similar patterns as those in Section~\ref{sec:analysis}.

HC-SMoE (right) exhibits severe routing collapse, visually manifested as static dark vertical banding. This pattern indicates that traffic is disproportionately funneled into a limited subset of experts, leading to severe load imbalance.

In contrast, PRIME (middle) closely aligns with the routing distribution of the original model (left) on downstream tasks. Even on this out-of-distribution task, PRIME still preserves the activation patterns of the uncompressed model. This consistency confirms that our frequency-aware global optimization achieves robust generalization beyond the calibration dataset.

\section{Effectiveness on Mixtral}
\label{app:mixtral}

\begin{table}[htbp]
\centering
\caption{Performance Comparisons on Mixtral-8$\times$7B. $K$ = retained experts per layer.}
\label{tab:main_results_mixtral}
\begin{small}
\begin{sc}
\resizebox{\textwidth}{!}{%
\begin{tabular}{c l cccc c c ccc c}
\toprule
$K$ & Method & BoolQ & MMLU & OBQA & RTE & MC Avg & Longbench & HE & MBPP & Coding Avg & GSM8K \\
\midrule
8 & Original & 85.05 & 67.12 & 35.00 & 71.12 & 64.57 & 35.25 & 41.46 & 60.85 & 51.15 & 70.43 \\
\midrule
\multirow{3}{*}{6} & HC-SMoE & 85.54 & 60.43 & 32.40 & 67.15 & 61.38 & \textbf{34.38} & 40.24 & 55.29 & 47.77 & 68.69 \\
& REAP & 85.60 & \textbf{63.08} & 31.40 & 68.95 & 62.26 & 30.25 & \textbf{40.85} & 57.41 & 49.13 & 65.88 \\
& \textbf{PRIME (Ours)} & \textbf{86.27} & 62.54 & \textbf{33.80} & \textbf{71.48} & \textbf{63.52} & 34.13 & \textbf{40.85} & \textbf{60.05} & \textbf{50.45} & \textbf{69.67} \\
\bottomrule
\end{tabular}
}
\end{sc}
\end{small}
\end{table}

We further evaluate PRIME on Mixtral-8$\times$7B to verify its efficacy on coarse-grained MoE architectures.
As presented in Table~\ref{tab:main_results_mixtral}, PRIME achieves the strongest aggregate performance among the compressed methods, while exhibiting some metric-specific trade-offs.
When retaining 6 experts (25\% reduction), PRIME achieves the highest MC average of 63.52, surpassing REAP (62.26) and HC-SMoE (61.38).
PRIME also leads on coding tasks (Coding Avg 50.45 vs.\ 49.13 for REAP and 47.77 for HC-SMoE) and mathematical reasoning (GSM8K 69.67 vs.\ 65.88 for REAP and 68.69 for HC-SMoE).
These results confirm that the proposed unified optimization framework remains effective even in regimes where expert redundancy is scarce, with only 8 experts per layer.

\section{Per-Layer Pruned Expert Profiling}
\label{app:pruning_profile}

\begin{table}[htbp]
\caption{Per-layer pruning vs.\ merging profile on Qwen3-30B-A3B ($K=96$).}
\label{tab:layer_profile}
\centering
\begin{small}
\begin{tabular}{l cc l}
\toprule
{\sc Layer} & {\sc Pruned} & {\sc Merged} & {\sc Pattern} \\
\midrule
0--5   & 10.5 & 21.5 & Mixed (merge-dominant) \\
6--11  & 18.3 & 13.7 & Mixed (prune-dominant) \\
12--23 & 9.1  & 22.9 & Merge-dominant \\
24--35 & 4.4  & 27.6 & Strongly merge-dominant \\
36--47 & 6.3  & 25.7 & Merge-dominant \\
\midrule
Total & $\approx$410 & $\approx$1126 & 73\% merged, 27\% pruned \\
\bottomrule
\end{tabular}
\end{small}
\end{table}

To understand how PRIME allocates its pruning and merging budget across layers, we profile the per-layer decisions on Qwen3-30B-A3B at $K=96$ (32 experts removed per layer). Table~\ref{tab:layer_profile} reports the average number of pruned vs.\ merged experts across layer groups.

Overall, 73\% of removed experts are merged and 27\% are pruned. The ILP adaptively selects pruning when expert pairwise distances are large (i.e., merging would introduce harmful representational averaging) and selects merging when experts are geometrically similar (preserving information through consolidation). This layer-adaptive mixture emerges naturally from the optimization and motivates the unified formulation: a merge-only strategy would force harmful averaging in layers with dissimilar experts (e.g., layers 6--11), while a prune-only strategy would discard recoverable information in layers where experts are highly similar (e.g., layers 24--35).

\section{Compression Overhead Comparison}
\label{app:overhead_comparison}

\begin{table}[htbp]
\caption{End-to-end compression overhead comparison on A800 GPUs. ``Route'' denotes router profiling on calibration data; ``Collect'' denotes expert output collection; ``Distance'' denotes pairwise distance computation; ``Solve/Cluster'' denotes the optimization or clustering step; ``Merge'' denotes weight averaging.}
\label{tab:overhead_comparison}
\centering
\begin{small}
\begin{sc}
\resizebox{\textwidth}{!}{%
\begin{tabular}{l l ccccc c}
\toprule
Model & Method & Route & Collect & Distance & Solve/Cluster & Merge & Total \\
\midrule
\multirow{3}{*}{\shortstack[l]{Qwen3-30B-A3B\\(128 experts)}}
& PRIME & 62s & 66s & 38s & 42s & 0.13s & 3.5min \\
& REAP & 66s & --- & --- & {<}0.01s & --- & 1.1min \\
& HC-SMoE & 64s & 69s & 43s & 993s & 0.15s & 19.5min \\
\midrule
\multirow{3}{*}{\shortstack[l]{Qwen3.5-122B-A10B\\(256 experts)}}
& PRIME & 295s & 303s & 46s & 288s & 0.72s & 15.5min \\
& REAP & 298s & --- & --- & 0.03s & --- & 5.0min \\
& HC-SMoE & 292s & 307s & 42s & 7920s & 0.77s & 142.7min \\
\bottomrule
\end{tabular}
}
\end{sc}
\end{small}
\end{table}

Table~\ref{tab:overhead_comparison} presents the end-to-end compression overhead of all methods profiled on A800 GPUs. REAP, as a pruning-only method that ranks experts solely by router-weighted activation norms without iterative optimization or weight merging, achieves the lowest overhead (1.1 minutes on 128 experts and 5.0 minutes on 256 experts). PRIME requires moderately more time due to the additional stages of expert output collection, pairwise similarity computation, and ILP solving, totaling 3.5 minutes on Qwen3-30B-A3B (128 experts) and 15.5 minutes on Qwen3.5-122B-A10B (256 experts). HC-SMoE is substantially slower, with its hierarchical clustering step consuming 993 seconds (20.7s/layer) at 128 experts and 7920 seconds (165s/layer) at 256 experts, resulting in a total overhead of 142.7 minutes---over 9$\times$ that of PRIME.

We emphasize that MoE compression is a \emph{one-time offline procedure} performed before deployment; the compressed model is then served indefinitely. From this perspective, PRIME's 15.5-minute overhead even on the largest 256-expert model is entirely practical. In exchange for approximately 10 additional minutes relative to REAP, PRIME achieves superior aggregate accuracy across most benchmarks and compression settings (cf.\ Tables in Section~\ref{sec:mainresults}), making this marginal cost trivially amortized over the serving lifetime of the model. Conversely, HC-SMoE's 142.7-minute overhead at the 256-expert scale---nearly 10$\times$ that of PRIME---becomes a meaningful practical barrier, particularly when iterating over compression configurations during development.

\end{document}